\documentclass[runningheads]{llncs}
\usepackage[T1]{fontenc}

\usepackage{color}
\usepackage{amsmath}
\usepackage{amssymb}
\usepackage{xcolor}
\usepackage{graphicx}
\usepackage{subcaption}
\usepackage{bbold}
\usepackage{booktabs}
\usepackage{multirow}
\usepackage{hyperref}
\usepackage{cleveref}

\begin{document}
	\title{Evaluating Machine Unlearning via Epistemic Uncertainty}
	
\author{Alexander Becker \and Thomas Liebig}
\authorrunning{A. Becker, T. Liebig}
\institute{TU Dortmund University, 44227 Dortmund, Germany \email{\{alexander2.becker,thomas.liebig\}@tu-dortmund.de}}
	
	\maketitle
	
	\begin{abstract}
		There has been a growing interest in Machine Unlearning recently, primarily due to legal requirements such as the General Data Protection Regulation (GDPR) and the California Consumer Privacy Act. Thus, multiple approaches were presented to remove the influence of specific target data points from a trained model. However, when evaluating the success of unlearning, current approaches either use adversarial attacks or compare their results to the optimal solution, which usually incorporates retraining from scratch. We argue that both ways are insufficient in practice. In this work, we present an evaluation metric for Machine Unlearning algorithms based on epistemic uncertainty. This is the first definition of a general evaluation metric for Machine Unlearning to our best knowledge.
		
		\keywords{Machine Unlearning  \and Forgetting \and Evaluation \and Trustworthy ML \and Data Protection.}
	\end{abstract}
	
	\section{Introduction and Related Work}
	In the last years, there has been an increasing interest in the domain of Machine Unlearning. Several approaches were presented, which apply to different models and come with different assumptions and restrictions. So far, these approaches cover unlearning in decision trees and random forests \cite{brophy2021machine,schelter2021hedgecut}, linear models such as logistic regression \cite{aldaghri2021coded,golatkar2020eternal,guo2019certified}, neural networks \cite{bourtoule2021machine,golatkar2021mixed,golatkar2020eternal,graves2020amnesiac,guo2019certified} and even Markov Chain Monte Carlo \cite{fu2022knowledge,nguyen2022markov}. While most approaches focus on forgetting\footnote{We use the terms unlearning and forgetting interchangeably.} in a single model, there also exist works that deal with federated models instead \cite{wang2021federated}. The vast majority of Machine Unlearning algorithms are motivated by legal requirements such as the General Data Protection Regulation (GDPR) or the California Consumer Privacy Act. However, it is worth mentioning that Machine Unlearning is not limited to this use case. In \cite{liu2022backdoor}, Liu et al., for instance, utilize forgetting in order to remove backdoors that were induced into a model. Since the field of Machine Unlearning is rather young, we believe that applications in many other domains will likely arise soon, e.g., model revision, continual learning and bias correction.
	
	When it comes to evaluation, the existing approaches can be divided into three categories. First, there are those approaches that provably guarantee perfect unlearning \cite{aldaghri2021coded,bourtoule2021machine,brophy2021machine,chen2022recommendation,golatkar2020eternal,neel2021descent,ullah2021machine}, and therefore do not need any evaluation. However, they often come with strong assumptions, making them only applicable in some specific scenarios. Others like the SISA (Sharded, Isolated, Sliced, Aggregated) training approach \cite{bourtoule2021machine} are generally applicable but come with considerable performance losses compared to a regularly trained model. Second, several approaches give guarantees on how well the forgetting works \cite{fu2022knowledge,gupta2021adaptive,neel2021descent,sekhari2021remember}. All of them argue about guarantees in the sense of differential privacy, i.e., the scrubbed model cannot be distinguished from a retrained model up to a certain probability. In the context of Machine Unlearning, this property is usually referred to as certified removal and was first introduced by Guo et al. \cite{guo2019certified}. Just as for the perfect forgetting approaches, they often come with strong assumptions.
	
	Finally, we have those approaches that are neither perfect nor give guarantees \cite{golatkar2021mixed,golatkar2020eternal,graves2020amnesiac,liu2022backdoor,schelter2021hedgecut,wang2022efficiently,wang2021federated}, but evaluate the success of unlearning purely empirical or compare the resulting model with an actual retrained model. The latter might be interesting from a theoretical point of view but is inapplicable in practice. However, measuring the success of forgetting purely empirical might give us a good hint whether unlearning works. A general approach for empirically evaluating the success of an unlearning algorithm is through adversarial attacks. Evaluation solely through adversarial attacks, however, comes with some clear disadvantages.
	On the one hand, performing different attacks is rather expensive since it requires training the corresponding attack models. Concretely, if the costs for forgetting \textit{and} evaluation exceed the costs for retraining from scratch, then the evaluation is not suitable in practice. On the other hand, we know that a failed attack is in no case a guarantee that unlearning really removed the information about the sensitive data. An attack might also fail due to other aspects such as badly chosen hyper-parameters, poorly estimated model parameters or the type and structure of the attack model. So to say, a failed attack on the scrubbed model can be seen as necessary but not as sufficient. Another point is that there might be new attacks in the future that can successfully gain sensitive information from the model, even though none of the currently existing attacks were able to do so. Nonetheless, the major advantage of those approaches is that they only come with few or even no assumptions, which makes them generally applicable, e.g., for deep neural networks.
	
	Given all of the concerns mentioned above, we utter the necessity of an evaluation metric that captures how well an unlearning algorithm removed the information about the sensitive data points and, at the same time, is theoretically profound and efficient enough for practical use.
	
	The contributions of this work are the following:
	\begin{enumerate}
		\item The definition of an evaluation metric that measures the success of Machine Unlearning algorithms. This metric is based on epistemic uncertainty and also allows an information theoretical interpretation.
		
		\item We give a theoretical upper bound for the metric that is more efficient to compute than the actual metric, which is particularly important for larger models and datasets.
		
		\item In the experiments, we investigate our evaluation metric on three forgetting approaches and discuss that it is useful for those approaches that change the model parameters in the direction of the retrained model.
	\end{enumerate}

	The rest of this paper is structured as follows. In \Cref{sec:preliminaries}, we will give a short introduction to the unlearning algorithms used in this work. We will then present our evaluation metric in \Cref{sec:metric}, including its intuition, derivation, and theoretical upper bound. Afterward, we outline the experimental setup and state the results in \Cref{sec:experiments}. Finally, we discuss the results and give possible research directions for future work in \Cref{sec:conclusion}.
	
	\section{Unlearning Algorithms} \label{sec:preliminaries}
	In the following, we give a short overview of the forgetting approaches used in this work.
	For this we assume an already trained model with parameters $\theta$ trained on a dataset $D = D_r \cup D_f$ with $D_f \cap D_r = \emptyset$. $D_f$ states the target data that should be forgotten and $D_r$ states the remaining data.
	
	\subsection{Retraining}\label{sec:retraining}
	In many works \cite{aldaghri2021coded,bourtoule2021machine,fu2022knowledge,golatkar2020eternal,guo2019certified,gupta2021adaptive,neel2021descent,wang2021federated}, retraining the model from scratch on the remaining data $D_r$ is considered the optimal solution, since it does not use the sensitive data points $D_f$ during training and achieves high performance on $D_r$. However, for large models and datasets retraining from scratch is computationally expensive, which is why this is often considered impractical. The fact that retraining is considered the optimal solution is deeply connected to the idea of certified removal (CR) \cite{guo2019certified} or differential privacy \cite{dwork2014algorithmic}, respectively. A forgetting algorithm $\mathcal{U}$ is considered an $\epsilon$-CR \cite{guo2019certified}, iff $\forall \mathcal{T} \subseteq \mathcal{H}, D \subseteq \mathcal{X}$ and $x \in D$:
	
	\begin{equation} \label{eq:ecr}
		e^{- \epsilon} \leq \frac{P(\mathcal{U} ( \mathcal{A}(D), D, x) \in \mathcal{T})}{P(\mathcal{A}(D \setminus \{x\}) \in \mathcal{T})} \leq e^\epsilon,
	\end{equation}

	where $\mathcal{A}$ is considered a learning algorithm, $\mathcal{H}$ the hypothesis space, $\mathcal{X}$ the data space, and $x$ an arbitrary data point from our dataset $D$. \Cref{eq:ecr} states that for any hypothesis space, dataset, and data point, the chance of obtaining the same result via forgetting and retraining should be equal with a tolerated margin of $\epsilon$.
	
	\subsection{Amnesiac Unlearning}
	Amnesiac Unlearning \cite{graves2020amnesiac} is an imperfect forgetting algorithm without guarantees that removes the influence of the target data points by actually reverting all parameter updates they are related to. To implement this, it is necessary to keep track of the parameter updates $\Delta_{\theta_{e, b}}$ in all training epochs $e$ and for each batch $b$. Additionally, the set of data points $D_{e,b}$ relating to the updates must be provided. Given the parameter updates and the related data points, Amnesiac Unlearning subtracts all the updates $\Delta_{\theta_{e, b}}$ from the model parameters, where $D_{e,b}$ and the target data points $D_f$ have at least one data point in common:

	\begin{equation}
		\mathcal{U}_{AU}(\theta, D_f) = \theta - \sum_{e=1}^{E}\sum_{b=1}^{B} \mathbb{1}[ D_f \cap D_{e, b} \neq \emptyset ]  \Delta_{\theta_{e, b}}.
	\end{equation}

	After forgetting, the updates that were subtracted and the corresponding data points no longer need to be held up.
	Also note that for Amnesiac Unlearning iterative and batch forgetting are equivalent, since $\forall D_f, D'_f \subseteq D$:
	
	\begin{equation} \label{eq:amnesiac_equiv}
		\mathcal{U}_{AU}(\mathcal{U}_{AU}(\theta, D_f), D'_f) = \mathcal{U}_{AU}(\theta, D_f \cup D'_f ).
	\end{equation}

	\subsection{Fisher Forgetting}
	Fisher Forgetting \cite{golatkar2020eternal} follows a different approach by hiding the difference between the given model and a model that could have been obtained by retraining from scratch. For this, Fisher Forgetting assumes both models to be close already, such that a normal distribution can describe the difference between them with variance $\sigma_h^2$. The goal is to add normal distributed noise to the parameters to hide this difference. The choice of the covariance matrix for this normal distributed noise is the key element of this approach. Here, we will only give its definition and a short explanation. For further details, we refer to the original work by Golatkar et al. \cite{golatkar2020eternal}.
	Fisher Forgetting is defined as
	\begin{equation} \label{eq:fisher_forgetting}
		\mathcal{U}_{FF}(\theta, D_f) = \theta + \alpha^{\frac{1}{4}}F^{-\frac{1}{4}} n,
	\end{equation}
	where $n \sim \mathcal{N}(0, I)$ is standard normal distributed noise, $F$ is the approximated Fisher Information matrix (see \Cref{eq:fim} below) of the model w.r.t. the remaining data and $\alpha = \lambda \sigma_h^2$. The hyper-parameter $\lambda$ trades off the loss on the remaining data and the difference between the scrubbed and a retrained model. This trade-off directly originates from their Forgetting Lagrangian \cite{golatkar2020eternal}. $\sigma_h^2$ is also treated as a hyper-parameter because it is generally unknown and cannot be computed efficiently. Since $\lambda$ and $\sigma_h^2$ always occur together as a multiplied factor, they are combined into a single hyper-parameter $\alpha$. Just as Amnesiac Unlearning, Fisher Forgetting is an imperfect unlearning approach without any guarantees. However, it is worth mentioning that it originates from generalizing a perfect forgetting approach, namely Optimal Quadratic Scrubbing \cite{golatkar2020eternal}.
	
	\section{Measuring the Success of Forgetting} \label{sec:metric}
	When it comes to evaluating Machine Unlearning algorithms, there are three aspects that are of importance.
	\begin{enumerate}
		\item The scrubbed model should contain as little information as possible about the target data $D_f$. (Efficacy)
		
		\item The scrubbed model should still perform well on the remaining data $D_r$. (Fidelity)
		
		\item The computation of the scrubbed model should be more efficient than retraining the prior model from scratch on the remaining data $D_r$. (Efficiency)
	\end{enumerate}
	These three aspects form the foundation for evaluating Machine Unlearning algorithms and are widely agreed on as they are stated explicitly or implicitly in many works in the domain \cite{bourtoule2021machine,brophy2021machine,chen2022recommendation,fu2022knowledge,golatkar2021mixed,golatkar2020eternal,graves2020amnesiac,schelter2021hedgecut,sekhari2021remember,ullah2021machine,wang2021federated}. Here, we make use of the terminology as stated by Warnecke et al. in \cite{warnecke2021machine}.
	
	Instead of evaluating all three aspects with a single metric, we argue that it is more reasonable to evaluate them separately. While this might seem obvious for the efficiency, it is not that clear for the efficacy and the fidelity. For example, Golatkar et al. formalize the unlearning problem by their Forgetting Lagrangian \cite{golatkar2020eternal}, which incorporates both the performance on the remaining data and the difference between the scrubbed and the retrained model. However, the efficacy and the fidelity should strictly be evaluated separately. Otherwise, it would be possible to receive a better evaluation by only improving the performance on the remaining data without scrubbing any information about the target data. Thus, a metric that combines both aspects might be misleading in its interpretation.
	Since there are plenty of ways to evaluate the fidelity of a model, e.g. by accuracy, we will focus on evaluating the efficacy in this work.

	\subsection{Evaluating Forgetting via Epistemic Uncertainty}
	
	Our evaluation metric is both motivated by information theory and epistemic uncertainty. Since the efficacy informally describes the goal to minimize residual information about the target data, we want to start by giving a more precise formalization of the term of information in this context. We think that the most suitable notion of information is that of the Fisher Information matrix (FIM) \cite{schervish2012theory}
	
	\begin{equation}\label{eq:fim}
		\mathcal{I}(\theta; D) = \mathbb{E}_\theta \left[ \sum_{x,y \in D} \frac{-\partial^2 \log p_\theta(y | x)}{\partial \theta \partial \theta^T} \right] \; ,
	\end{equation}
		
	since it already measures the amount of information the model parameters $\theta$ carry about the dataset $D$. For practical reasons, we further assume an empirical approximation of the FIM, namely the Levenberg-Marquart approximation \cite{martens2014new}
	\begin{equation}\label{eq:fim_approx}
		\mathcal{I}(\theta; D) \simeq \frac{1}{|D|} \sum_{x, y \in D} \left(\frac{\partial \log p_\theta(y|x)}{\partial \theta}\right)^2.
	\end{equation}
	Approximating the FIM is often necessary in practice, since the computation of the second derivatives w.r.t. all parameter combinations and for each data point separately is computationally expensive, i.e. $\mathcal{O}(|D| \cdot F + |D| \cdot G \cdot |\theta|^2)$. Here $F$ and $G$ denote the runtime complexities for the model inference $p_\theta(y|x)$ and the gradient computations, respectively. The FIM might not even fit in memory for larger models such as deep neural networks. Keeping in mind that forgetting and its evaluation must be more efficient than retraining from scratch, computing the whole FIM is not reasonable.
	
	Note that the approximation in  \Cref{eq:fim_approx}, which can be done in $\mathcal{O}(|D| \cdot F +  |D| \cdot G \cdot |\theta|)$, will give us a diagonal matrix or a vector of length $|\theta|$, respectively. For all $i \in [1, |\theta|]$ the $i$-th entry states the amount of information $\theta_i$ carries about the dataset $D$. By computing the trace of $\mathcal{I}(\theta; D)$, we obtain an overall information value
	
	\begin{equation}\label{eq:uncertainty}
		\imath(\theta; D) = tr(\mathcal{I}(\theta; D)) =  \frac{1}{|D|} \sum_{i = 1}^{|\theta|} \sum_{x, y \in D} \left(\frac{\partial \log p_\theta(y|x)}{\partial \theta_i}\right)^2.
	\end{equation}
	
	The trace of the FIM can not only be seen as the total amount of information, but also as the epistemic uncertainty of $\theta$ with respect to $D$ \cite{hullermeier2021aleatoric}. Thus, $\imath(\theta; D)$ can be interpreted in the following ways:
	\begin{enumerate}
		\item The amount of information the model parameters carry about the given dataset.
		\item The degree of how much the model parameters can vary, while still describing the dataset equally well.
		\item The epistemic uncertainty of the model parameters with respect to the dataset.
	\end{enumerate}
	
	Given the information respectively uncertainty score in \Cref{eq:uncertainty}, we see that it is also sufficient to compute the FIM diagonal only since none of the other matrix entries are required. We argue that scrubbing information from a model should always increase the epistemic uncertainty and decrease the amount of information about the target data. Otherwise, a Streisand effect \cite{jansen2015streisand} may occur. Therefore, we introduce the efficacy score
	\begin{equation}\label{eq:efficacy}
		\text{efficacy}(\theta; D) = \begin{cases}
			\frac{1}{\imath(\theta; D)}, &\text{if } \imath(\theta; D) > 0 \\
			\infty, &\text{otherwise}
		\end{cases}
	\end{equation}
	as a measure for how much information the model exposes. The model converges to the optimum of 0 with increasing uncertainty.
	However, we argue that in most cases, even exact unlearning algorithms or retraining will not achieve an efficacy score of (approximately) 0 due to the generalization from the remaining data points. Although it is hard to tell what the actual optimum efficacy is, the efficacy score is still helpful to compare the results of multiple unlearning algorithms since we still know that the lower the score, the higher the uncertainty about the target data.
	
	With an increasing number of target data points and model parameters, the efficacy score becomes more computationally expensive since it requires the computation of the gradients w.r.t. the model parameters for each target data point separately. In the following, we address this problem by defining a theoretical upper bound for the efficacy score. For this, we first define the cross-entropy loss
	\begin{equation}
		\mathcal{L}(\theta, D) = \frac{1}{|D|} \sum_{x, y \in D} -\log p_\theta(y | x)
	\end{equation}
	with gradients
	\begin{equation}
		\nabla \mathcal{L}(\theta, D) = \frac{1}{|D|} \sum_{x, y \in D} \frac{-\partial \log p_\theta(y | x)}{\partial \theta}
	\end{equation}
	and $\ell_2$-norm of the gradients
	\begin{equation}
		\| \nabla \mathcal{L}(\theta, D) \|_2 = \left(\sum_{\theta_i \in \theta} \frac{1}{|D|^2} \left(\sum_{x, y \in D} \frac{\partial \log p_\theta(y | x)}{\partial \theta_i}\right)^2\right)^{\frac{1}{2}}.
	\end{equation}
	
	Then,  \Cref{th:uncertainty_lower_bound} shows that the uncertainty is lower bounded by the squared gradient norm of the cross-entropy loss.
	
	\begin{theorem}\label{th:uncertainty_lower_bound}
		Let $\mathcal{L}(\theta, D)$ be the cross-entropy loss.
		The squared gradient norm of the cross-entropy loss  forms a lower bound for the information score:
		\begin{equation}
			\| \nabla \mathcal{L}(\theta, D) \|_2^2 \leq \imath(\theta; D)
		\end{equation}
	\end{theorem}
	
	\begin{proof}[\Cref{th:uncertainty_lower_bound}]
		\begin{align*}
			\imath(\theta; D) &= \frac{1}{|D|} \sum_{\theta_i \in \theta} \sum_{x,y \in D} \left(\frac{\partial \log p_\theta(y|x)}{\partial \theta_i}\right)^2  & \text{(\Cref{eq:uncertainty})}\\
			&\geq \frac{1}{|D|} \sum_{\theta_i \in \theta} \frac{1}{|D|} \left(\sum_{x,y \in D} \frac{\partial \log p_\theta(y|x)}{\partial \theta_i}\right)^2 & \text{(*)}\\
			&= \sum_{\theta_i \in \theta} \frac{1}{|D|^2} \left(\sum_{x,y \in D} \frac{\partial \log p_\theta(y|x)}{\partial \theta_i}\right)^2 \\
			&= \| \nabla \mathcal{L}(\theta, D) \|_2^2
		\end{align*}
		(*) Cauchy-Schwarz inequality \cite{steele2004cauchy}
	\end{proof}

	From the uncertainty lower bound in \Cref{th:uncertainty_lower_bound}, we directly obtain the theoretical upper bound for the efficacy score (\Cref{lem:efficacy_upper_bound}).
	
	\begin{lemma}\label{lem:efficacy_upper_bound}
		Let $\mathcal{L}(\theta, D)$ be the cross-entropy loss.
		If $\imath(\theta; D) > 0$, then the efficacy score is upper bounded by
		\begin{equation}
			\textnormal{efficacy}(\theta; D) \leq \frac{1}{\| \nabla \mathcal{L}(\theta, D) \|_2^2}.
		\end{equation}
	\end{lemma}

	In contrast to the efficacy score, its upper bound is far less computationally expensive since it only requires computing the gradients w.r.t. the parameters once, instead of $|D|$ times. Thus, the efficacy upper bound can be computed in $\mathcal{O}(|D| \cdot F + G \cdot |\theta|)$ in contrast to $\mathcal{O}(|D| \cdot F + |D| \cdot G \cdot |\theta|)$ for the efficacy.
	
	\section{Experiments} \label{sec:experiments}
	Our experiments aim to study the practical usefulness of the efficacy score and its theoretical upper bound stated in \Cref{lem:efficacy_upper_bound}. For this, we formulate the following two hypotheses:
	\begin{enumerate}
		\item Applying a forgetting algorithm to a trained model always reduces the efficacy score for the target data.
		\item A lower efficacy score always implies a less successful adversarial attack.
	\end{enumerate}
	In order to test the above hypotheses, we consider the following experimental setup.
	
	\subsection{Experimental Setup}
	We define a simple neural network architecture, consisting of five fully connected layers with $N$, 512, 256, 128, and 10 neurons per layer, respectively. $N$ denotes the input dimension that varies depending on the input data. All layers but the output layer are ReLU activated. We use the softmax function for the last layer to obtain an output distribution over all classes. For the datasets, we use the well-known MNIST \cite{lecun1998gradient} and CIFAR10 \cite{krizhevsky2009learning} datasets for image classification. We only use a small subset of both datasets to reduce the applied forgetting algorithms' runtime and memory consumption. Thus, we only use the first 100 samples for each class. All training examples are transformed to greyscale images, and pixel values are normalized to $[0, 1]$.
	
	We train the model on the MNIST dataset for 50 epochs with a learning rate of 0.1 and a batch size of 32 using standard stochastic gradient descent (SGD). For CIFAR10, we train the model for 200 epochs with a learning rate of 0.1 and a batch size of 64, again using standard SGD. Each training is repeated 20 times with different random initializations, giving us a total number of 40 pre-trained models.
	For the target data we want to forget, we chose a fixed class for both datasets: Class 3 for MNIST and class 8 for CIFAR10. Both classes are chosen arbitrarily. Furthermore, we define a percentage value for how many data points of the target class we would like to forget. In all experiments the percentage values are 0.01, 0.1, 0.25, 0.5, 0.8 and 1, respectively.
	
	As a baseline for how successful an adversarial attack should be at most after forgetting, we perform a membership inference attack (MIA) \cite{shokri2017membership} on all 40 pre-trained models for each of the above percentage values. This leads to a total number of 240 model attacks. The goal of a MIA is to predict whether a data point is part of the training dataset of the target model. Therefore, the attack model only has black-box access to the target model. Here we make use of the MIA implementation in IBM's Adversarial Robustness Toolbox \cite{art2018}. Each attack model is trained with default parameters on the same training set as its target model.
	
	We consider three different approaches for forgetting algorithms: Retraining, Fisher Forgetting, and Amnesiac Unlearning. Thereby, retraining is considered the optimal solution (see \Cref{sec:retraining}). All forgetting approaches are applied to each pre-trained model using the same target classes and percentage values as the MIAs. This again leads to a total of applying each forgetting approach 240 times. Retraining is always performed with the same model initialization and hyper-parameters as the original training. Our implementation of Fisher Forgetting exactly follows the implementation details stated in the extended version of \cite{golatkar2020eternal}, which slightly differs from the definition of Fisher Forgetting given in \Cref{eq:fisher_forgetting}. The implementation details are purely empirical and arose from the experiments in which these led to better results.
	
	Finally, we want to note that all our experiments were implemented using PyTorch \cite{paszke2019pytorch} and are publicly available on GitHub\footnote{\url{https://github.com/ROYALBEFF/evaluating\_machine\_unlearning\_via\_epistemic\_uncertainty}}.
	
	\subsection{Results}
	In the following, we present our experimental results on the MNIST dataset. We omit our experimental results on CIFAR10, since they do not provide any additional insights. The observation for both datasets are identical. \footnote{Plots and tables of the CIFAR10 experiments are included in the appendix.}
	
	First of all, we take a look at the efficacy values of the pre-trained models in comparison to the efficacy after forgetting (\Cref{fig:efficacy}). We observe that the efficacy increases on average for the pre-trained models and varies more with the fewer data points we take into account. For a single data point of the target class, i.e., the percentage equals 0.01, the observed efficacy values range from $10^{-7}$ to $10^{13}$, which exceeds the lowest and highest values for all other percentages. Note that for reasons of readability, the efficacy values for a single data point are omitted in \Cref{fig:efficacy}.
	Altogether training increases the efficacy of the target data, as can be seen in \Cref{fig:efficafy_comparison}.

	\begin{figure}
		\centering
		\begin{subfigure}[b]{0.5\textwidth}
			\centering
			\includegraphics[scale=0.43]{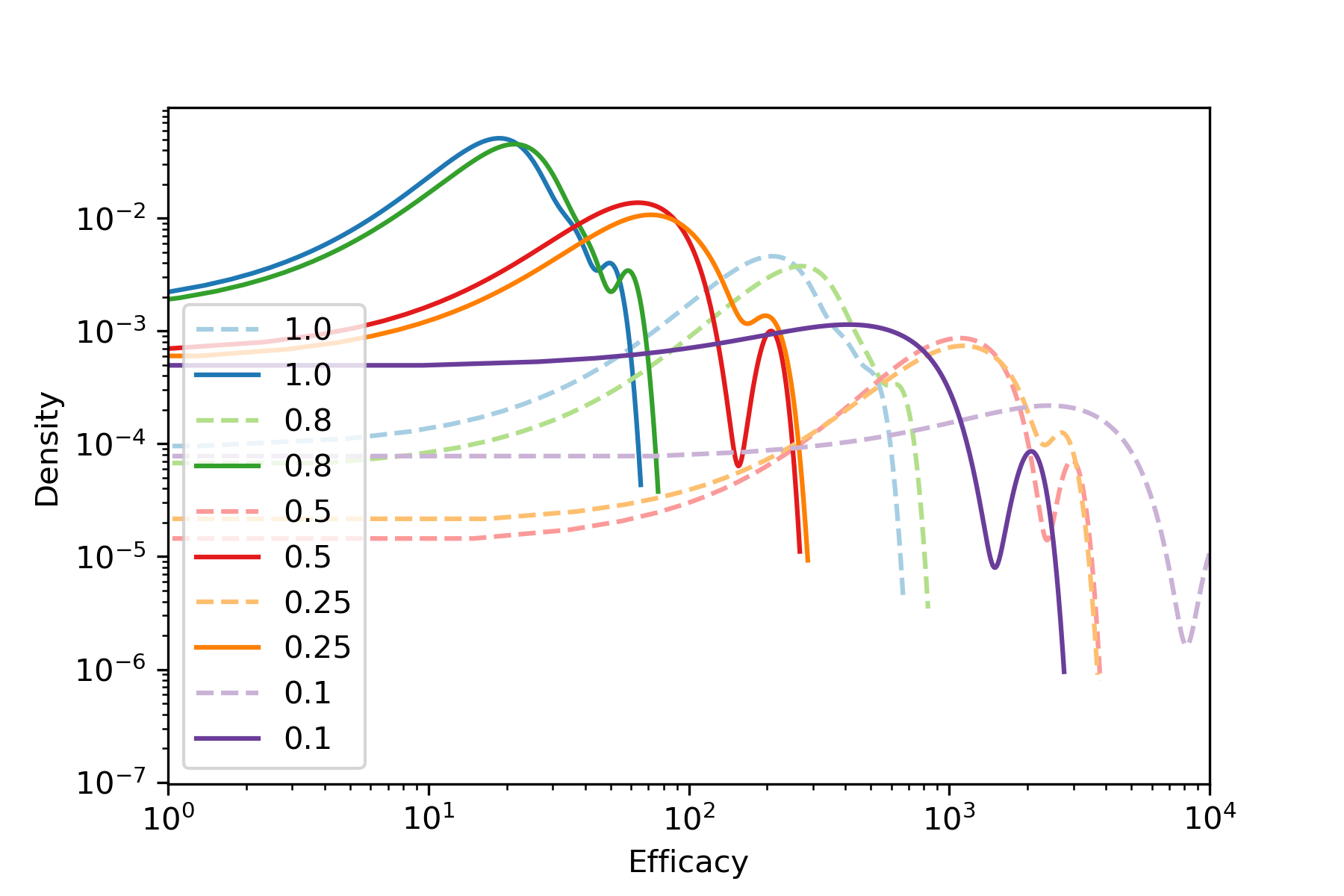}
			\caption{Pre-trained}
		\end{subfigure}\hfill
		\begin{subfigure}[b]{0.5\textwidth}
			\centering
			\includegraphics[scale=0.43]{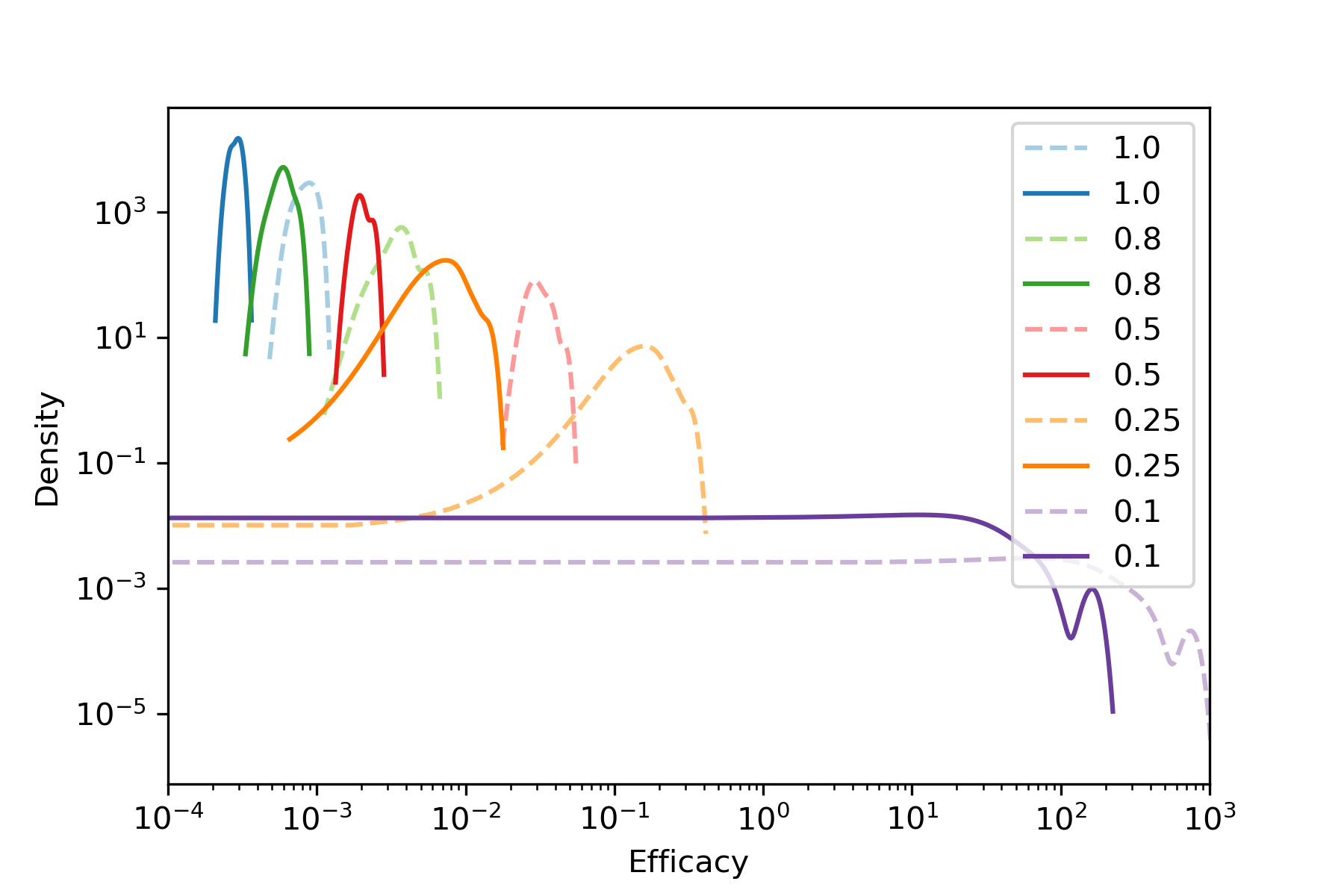}
			\caption{Retraining}
		\end{subfigure}
		\begin{subfigure}[b]{0.5\textwidth}
			\centering
			\includegraphics[scale=0.43]{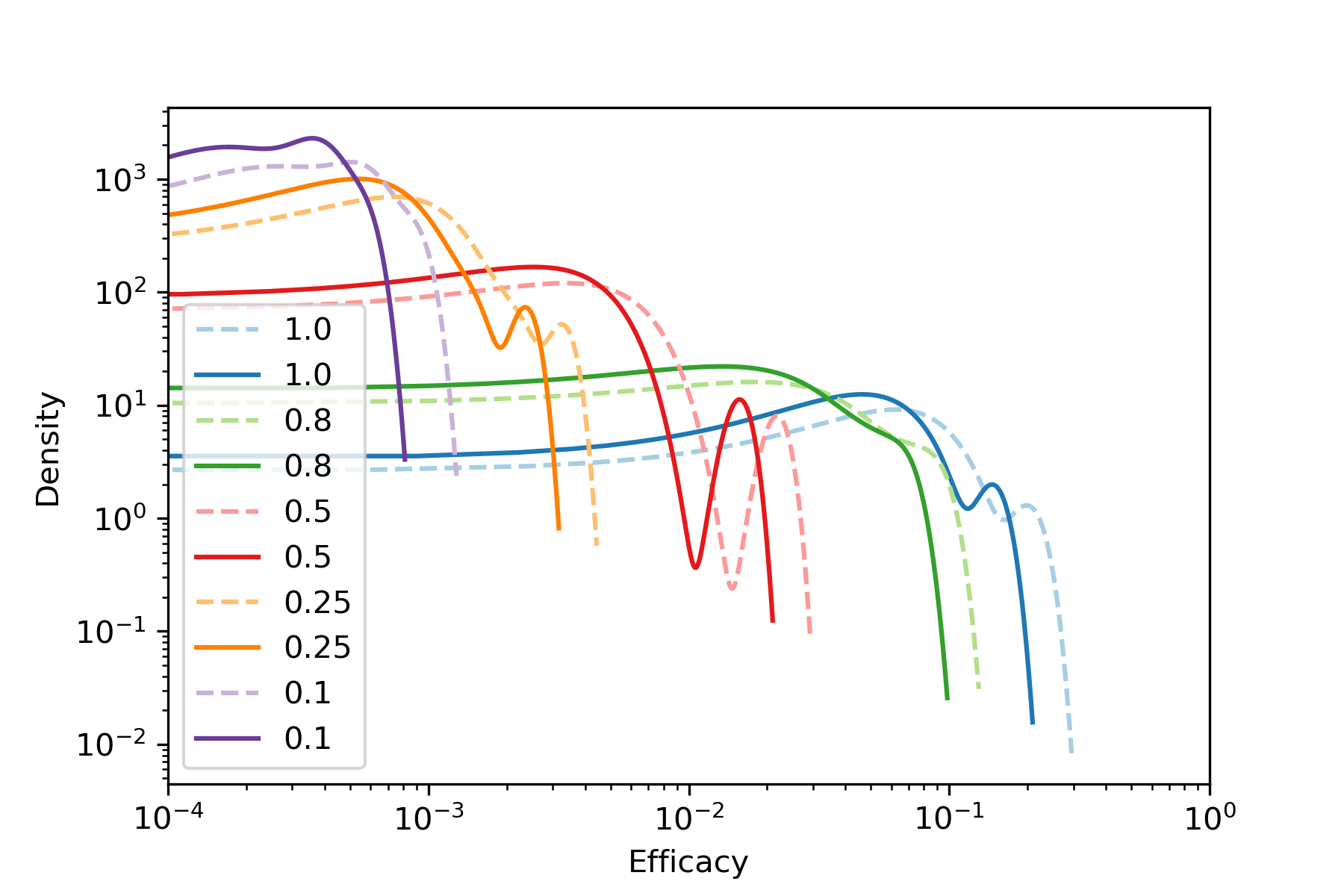}
			\caption{Amnesiac Unlearning}
			\label{fig:efficacy_amnesiac}
		\end{subfigure}\hfill
		\begin{subfigure}[b]{0.5\textwidth}
			\centering
			\includegraphics[scale=0.43]{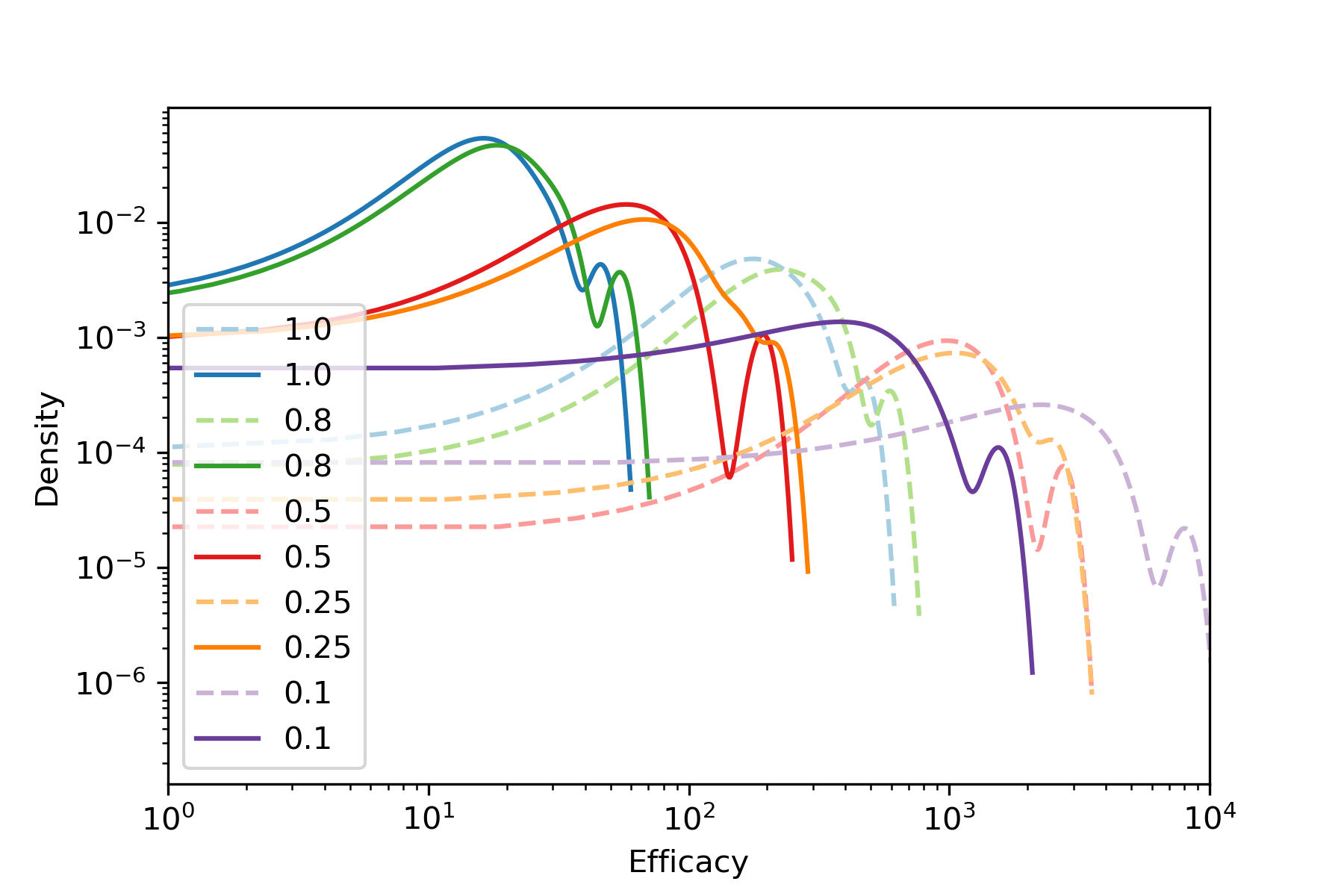}
			\caption{Fisher Forgetting}
			\label{fig:efficacy_fisher}
		\end{subfigure}
		\caption{Distributions of efficacy scores (solid lines) and upper bounds (dashed lines) over all pre-trained models trained on the MNIST dataset (a) before and (b)-(d) after forgetting. Each distribution corresponds to a percentage of the target class. For reasons of readability we omit the percentage of 0.01. Both axes are log scaled.}
		\label{fig:efficacy}
	\end{figure}

	\begin{figure}
		\centering
		\includegraphics[scale=0.6]{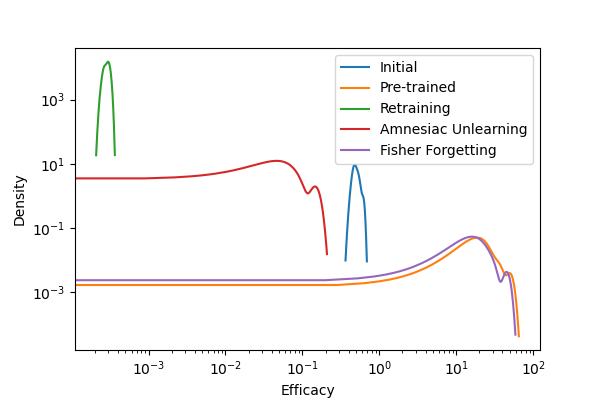}
		\caption{Efficacy comparison w.r.t. the whole target class before training (Initial), after training (Pre-trained) and after forgetting (Retraining, Amnesiac Unlearning, Fisher Forgetting). Both axes are log scaled.}
		\label{fig:efficafy_comparison}
	\end{figure}
	
	Forgetting about the target data decreases the efficacy in all our experiments, except for forgetting a single data point only. Here the efficacy almost remains unchanged, and in the case of Amnesiac Unlearning, we were even able to observe an increase over the pre-trained models. When forgetting more than just a single data point, we see that the efficacy continually decreases the more data we forget from our target class. This holds for both retraining from scratch and Fisher Forgetting. Interestingly enough, we observe the opposite behavior for Amnesiac Unlearning.
	In comparison to the pre-trained model, the efficacy values decrease, but a higher percentage yields an increase in the efficacy (\Cref{fig:efficacy_amnesiac}). Amnesiac Unlearning forgets about target data points by reverting those updates from training, which they directly influence, so this behavior is expected. With an increasing number of target data points, the number of updates that must be reverted increases as well. Therefore, the model further develops in the direction of the initial model. As a consequence, the efficacy values converge towards those of the initial model as well (see \Cref{fig:efficacy_amnesiac,fig:efficafy_comparison}). Finally, we see that Fisher Forgetting only slightly decreases the efficacy independent of the target data percentage. This holds for all values of $\alpha$ (see \Cref{eq:fisher_forgetting}) we tried in our experiments and can be traced back to clipping the values in the Fisher approximation, which is one of the implementation details mentioned in \cite{golatkar2020eternal}. This also leads to the fact that the scrubbed model can still classify all target data points correctly, which is undesirable in general. At the same time, we observe that the accuracy of the target data points decreases for the other approaches the higher the percentage (\Cref{tab:acc}).
	
	For the efficacy upper bounds (\Cref{fig:efficacy}), we observe that both the distribution shapes and the relations between the different distributions are preserved. This shows that the efficacy upper bound is suitable for comparing multiple forgetting results even though the absolute values are an order of magnitude larger than the actual efficacies.
	
	\begingroup
	\setlength{\tabcolsep}{10pt}
	\begin{table}
		\centering
		\begin{tabular}{c c c c c}
			\toprule
			\multirow{2}{*}{Model} & \multirow{2}{*}{$p$} & \multicolumn{3}{c}{Accuracy} \\
			&& $D_r$ & $D_f$ & $D_{test}$ \\ \midrule
			Pre-trained & * & 1.00 ± 0.00 & 1.00 ± 0.00 & 0.87 ± 0.00 \\ \midrule
			\multirow{6}{*}{Retraining} & 1 & 1.00 ± 0.00 & 0.00 ± 0.00 & 0.80 ± 0.00 \\
			& 0.8 & 1.00 ± 0.00 & 0.57 ± 0.05 & 0.84 ± 0.00 \\
			& 0.5 & 1.00 ± 0.00 & 0.89 ± 0.02 & 0.86 ± 0.00 \\ 
			& 0.25 & 1.00 ± 0.00 & 0.96 ± 0.01 & 0.87 ± 0.00 \\ 
			& 0.1 & 1.00 ± 0.00 & 1.00 ± 0.00 & 0.87 ± 0.00 \\ 
			& 0.01 & 1.00 ± 0.00 & 1.00 ± 0.00 & 0.87 ± 0.00 \\ \midrule
			\multirow{6}{*}{Amnesiac Unlearning} & 1 & 0.13 ± 0.03 & 0.00 ± 0.00 & 0.11 ± 0.02 \\
			& 0.8 & 0.12 ± 0.03 & 0.00 ± 0.00 & 0.11 ± 0.02 \\ 
			& 0.5 & 0.11 ± 0.01 & 0.00 ± 0.00 & 0.10 ± 0.01 \\ 
			& 0.25 & 0.16 ± 0.07 & 0.00 ± 0.00 & 0.15 ± 0.07 \\ 
			& 0.1 & 0.19 ± 0.08 & 0.00 ± 0.00 & 0.18 ± 0.07 \\
			& 0.01 & 0.60 ± 0.21 & 0.50 ± 0.50 & 0.53 ± 0.18 \\ \midrule
			Fisher Forgetting & * & 1.00 ± 0.00 & 1.00 ± 0.00 & 0.87 ± 0.00 \\ \toprule
		\end{tabular}

		\caption{Mean accuracy and standard deviation of all MNIST models on the remaining data $D_r$, the target data $D_f$ and the test data $D_{test}$. $p$ denotes the percentage of the target data and $*$ indicates that the accuracy values are the same over all percentages.}
		\label{tab:acc}
	\end{table}
	\endgroup
	
	Next, we take a look at the mean probabilities of the membership inference attacks performed on the pre-trained and the scrubbed models as shown in \Cref{fig:mean_mia}.
	We compute the mean probability over all target data points for each attack. The effectiveness of the MIA on the pre-trained model thereby forms the baseline, which is not to be exceeded. Otherwise, forgetting would expose information about the target data rather than removing it. On the other hand, we have a second baseline, which is given by the effectiveness of the MIA on the retrained model since the target data points were not used during training. However, due to correlations and generalization, an attack might still be successful, especially if not all data points of the target class were to be forgotten. 
	Our results show that forgetting always decreases the mean MIA probabilities and that the more data points we forget, the lower they get. Here we want to point out that this also holds for Amnesiac Unlearning, even though it increases the efficacy with increasing target data points.
	
	\begin{figure}[ht]
		\centering
		\begin{subfigure}[b]{0.5\textwidth}
			\centering
			\includegraphics[scale=0.43]{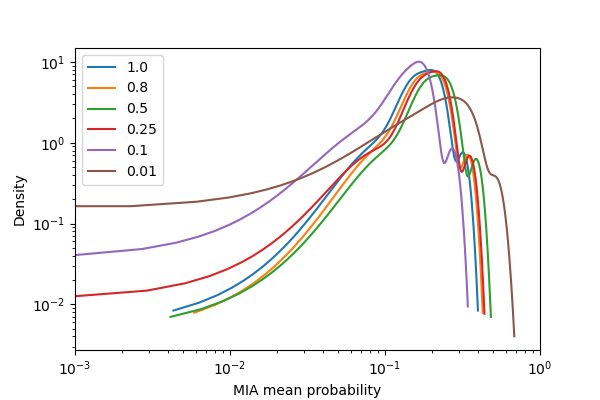}	
			\caption{Pre-trained}
		\end{subfigure}\hfill
		\begin{subfigure}[b]{0.5\textwidth}
			\centering
			\includegraphics[scale=0.43]{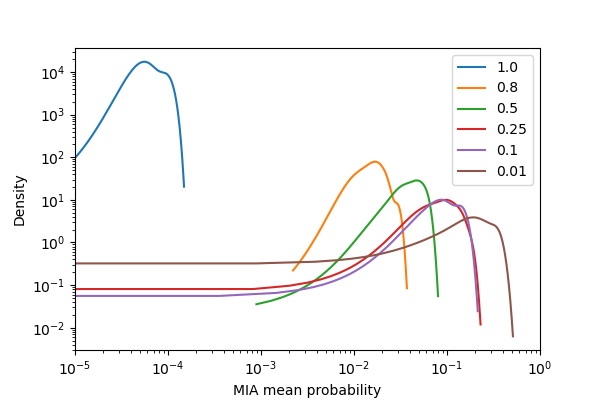}
			\caption{Retraining}
		\end{subfigure}
		\begin{subfigure}[b]{0.5\textwidth}
			\centering
			\includegraphics[scale=0.43]{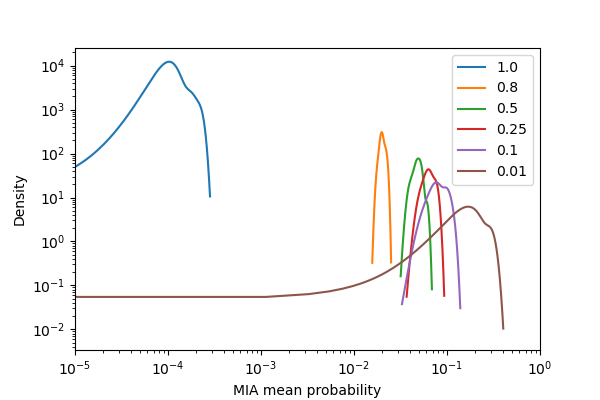}
			\caption{Amnesiac Unlearning}
		\end{subfigure}\hfill
		\begin{subfigure}[b]{0.5\textwidth}
			\centering
			\includegraphics[scale=0.43]{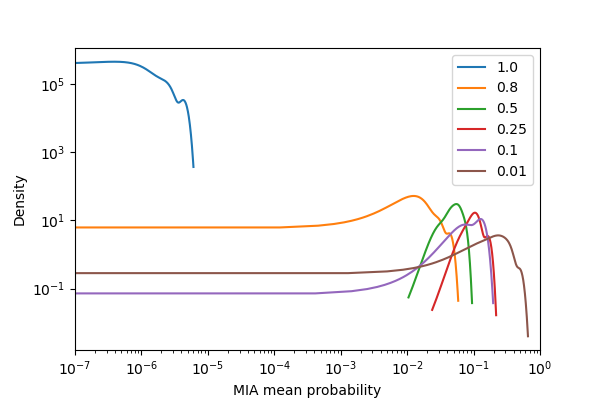}
			\caption{Fisher Forgetting}
		\end{subfigure}
		\caption{Distributions of membership inference attack mean probabilities over all pre-trained models trained on the MNIST dataset (a) before and (b)-(d) after forgetting. Each distribution corresponds to a percentage of the target class. Both axes are log scaled.}
		\label{fig:mean_mia}
	\end{figure}
	
	In \Cref{fig:efficacy_mia}, we illustrate how efficacy relates to the effectiveness of the MIAs. The efficacy values and the MIA mean probabilities do not vary much for the pre-trained model, even though the efficacy values are much larger when computed for a single data point only. 
	For retraining and Fisher Forgetting, we observe that with an increased efficacy score, the mean probabilities of the MIAs also become larger. Just as for the efficacy values themselves, we observe the opposite effect for forgetting via Amnesiac Unlearning. Here lower efficacies yield higher average MIA probabilities. Again, this can be traced back to the model converging to the initial model rather than to the pre-trained model as more data points are forgotten. Also, note that the efficacy varies over a large range of values for single data points while having similar MIA mean probabilities.
	
	\begin{figure}[ht]
		\centering
		\begin{subfigure}[b]{0.5\textwidth}
			\centering
			\includegraphics[scale=0.43]{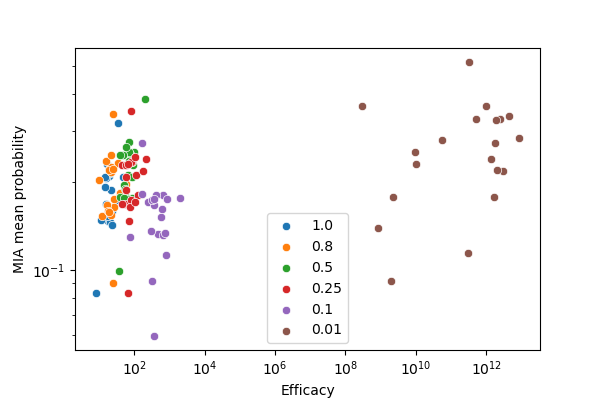}
			\caption{Pre-trained}
		\end{subfigure}\hfill
		\begin{subfigure}[b]{0.5\textwidth}
			\centering
			\includegraphics[scale=0.43]{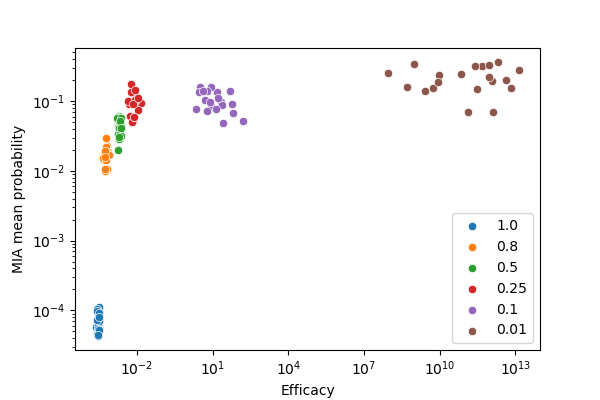}
			\caption{Retraining}
		\end{subfigure}
		\begin{subfigure}[b]{0.5\textwidth}
			\centering
			\includegraphics[scale=0.43]{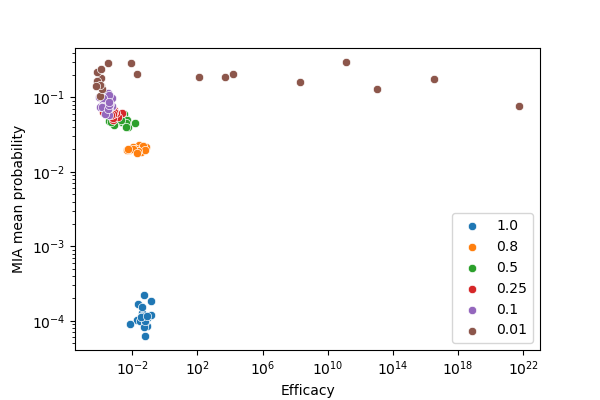}
			\caption{Amnesiac Unlearning}
		\end{subfigure}\hfill
		\begin{subfigure}[b]{0.5\textwidth}
			\centering
			\includegraphics[scale=0.43]{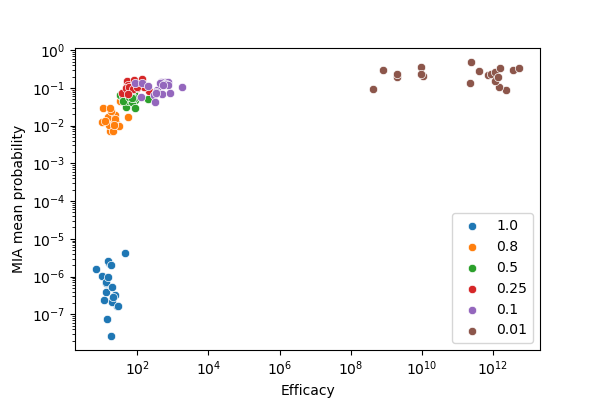}
			\caption{Fisher Forgetting}
		\end{subfigure}
		\caption{Log-log plot showing the relation between the efficacy and the membership inference attack mean probability (a) before and (b)-(d) after forgetting.}
		\label{fig:efficacy_mia}
	\end{figure}
	
	With the results given above, we revisit our hypotheses. Even though most experiments give evidence for the first hypothesis, we have to reject it due to our experiments on Amnesiac Unlearning. Here, forgetting first decreases the efficacy, but the more data points we forget, the more the efficacy increases. This is especially problematic since there is no difference between batch forgetting and iterative forgetting for Amnesiac Unlearning (\Cref{eq:amnesiac_equiv}). Thus, forgetting all target points consecutively increases the efficacy. Likewise, we have to reject our second hypothesis. Our results from forgetting through retraining and Fisher Forgetting support this hypothesis. However, the results from forgetting through Amnesiac Unlearning show the exact opposite behaviour.
	
	\section{Conclusion and Future Work} \label{sec:conclusion}
	 In this work, we presented a metric for evaluating the success of forgetting algorithms and formulated two hypotheses, which we tested in our experiments. This metric has many advantages compared to existing evaluation methods since it can be computed efficiently, allows an information theoretical interpretation, does not require retraining the model, and follows cognitive considerations regarding the relationship between uncertainty and forgetting. Moreover, for its theoretical upper bound, it even preserves the relation of the efficacy distributions over all percentage values. All of those advantages are crucial for practical applications.
	 Our results clearly show that the evaluation of forgetting algorithms is a complex task that cannot simply be solved by measuring accuracy or performing adversarial attacks. Decreasing the accuracy of the target data points and preventing adversarial attacks from being successful can be seen as necessary conditions but are not sufficient to guarantee that sensitive information was really scrubbed from the model. The former can be observed in our experiments using Fisher Forgetting (\Cref{tab:acc}, \Cref{fig:mean_mia}), while the latter can easily be illustrated using a small example. Consider a neural network consisting of a feature extraction part, followed by a classifier part. Reinitializing the classifier part would degenerate the model's performance. However, since the feature extraction part has not changed, the overall model still contains information about the training data.
	 
	 Even though we had to reject both hypotheses, we see that both retraining from scratch as the optimal solution and Fisher Forgetting give evidence for them. Merely our results of using Amnesiac Unlearning lead to a rejection of the hypotheses. As a crucial difference of the forgetting algorithms leading to this rejection, we identified the direction in which forgetting changes the model parameters. While retraining is considered the optimal solution and Fisher Forgetting aims to blur the difference between the model and the retrained model by adding noise, Amnesiac Unlearning lets the model converge towards the initial model before training (\Cref{fig:efficafy_comparison,fig:efficacy_amnesiac}). Thus, the directions in which the parameters are changed are fundamentally different. This implies that multiple evaluation metrics might be necessary that depend on the way the algorithms remove information from models. Therefore, we claim that the here presented efficacy metric is a first step in the direction of more general evaluation metrics for Machine Unlearning.
	 
	 Given the insights from this work, we want to do a larger survey on evaluating Machine Unlearning algorithms in the future, where we categorize the algorithms depending on how and in which direction the model parameters are updated. In this context, we will further  study the here presented efficacy metric for those algorithms that aim to obtain a model close to a retrained one, since the results of our experiments look quite promising. Finally, finding the relation between the efficacy and certified removal is also an important direction for future work, since this will allow relating the metric to privacy guarantees in the sense of differential privacy.
	
	\section*{Acknowledgment}
	This work is supported by the Federal Ministry of Education and Research of Germany as part of the competence center for machine learning ML2R (01–S18038A).
	
	\bibliographystyle{splncs04}
	\bibliography{literature}

\newpage
\section*{Appendix}
Results of experiments on CIFAR10.

\begin{figure}
	\centering
	\begin{subfigure}[b]{0.5\textwidth}
		\centering
		\includegraphics[scale=0.43]{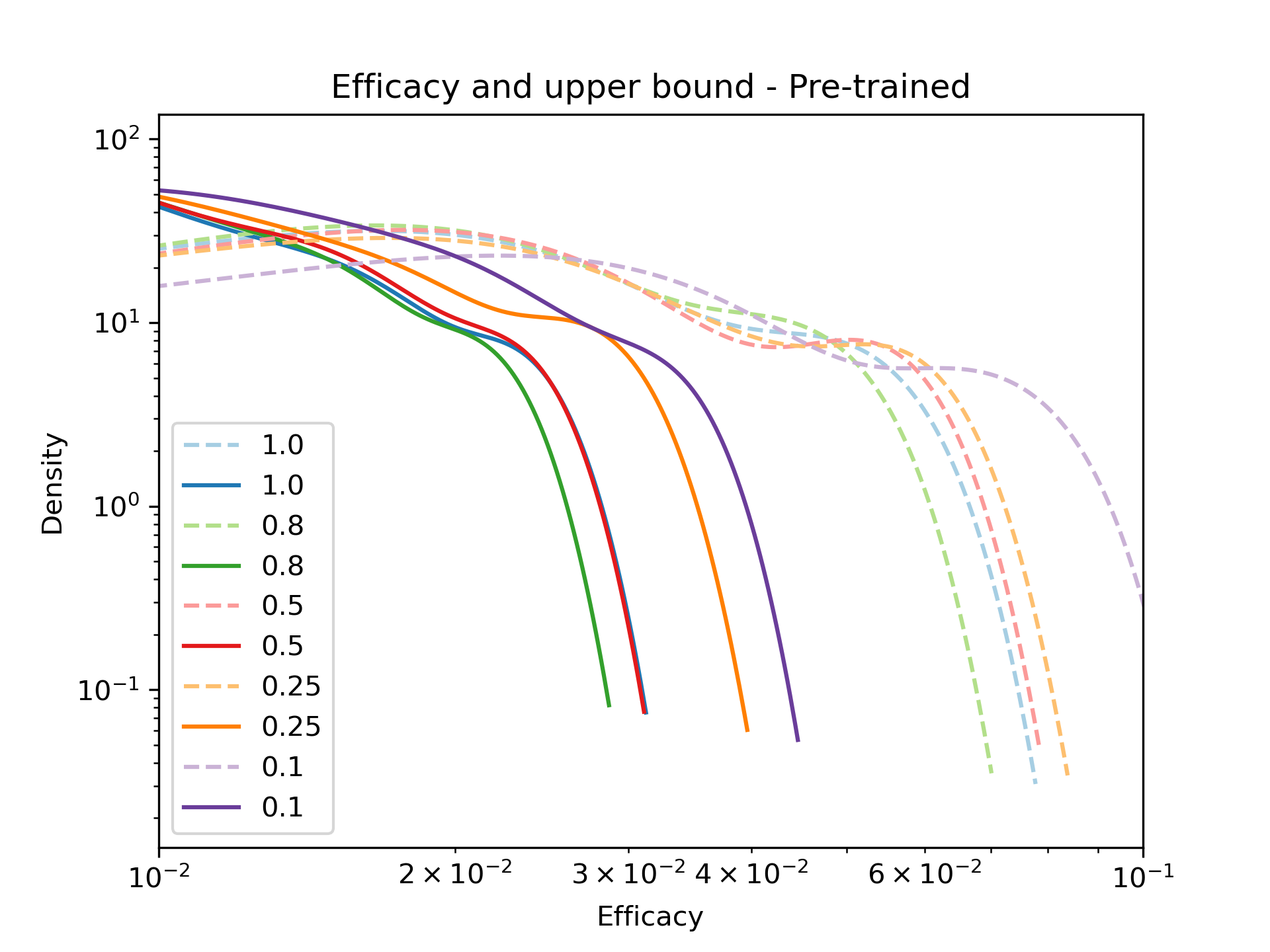}
		\caption{Pre-trained}
	\end{subfigure}\hfill
	\begin{subfigure}[b]{0.5\textwidth}
		\centering
		\includegraphics[scale=0.43]{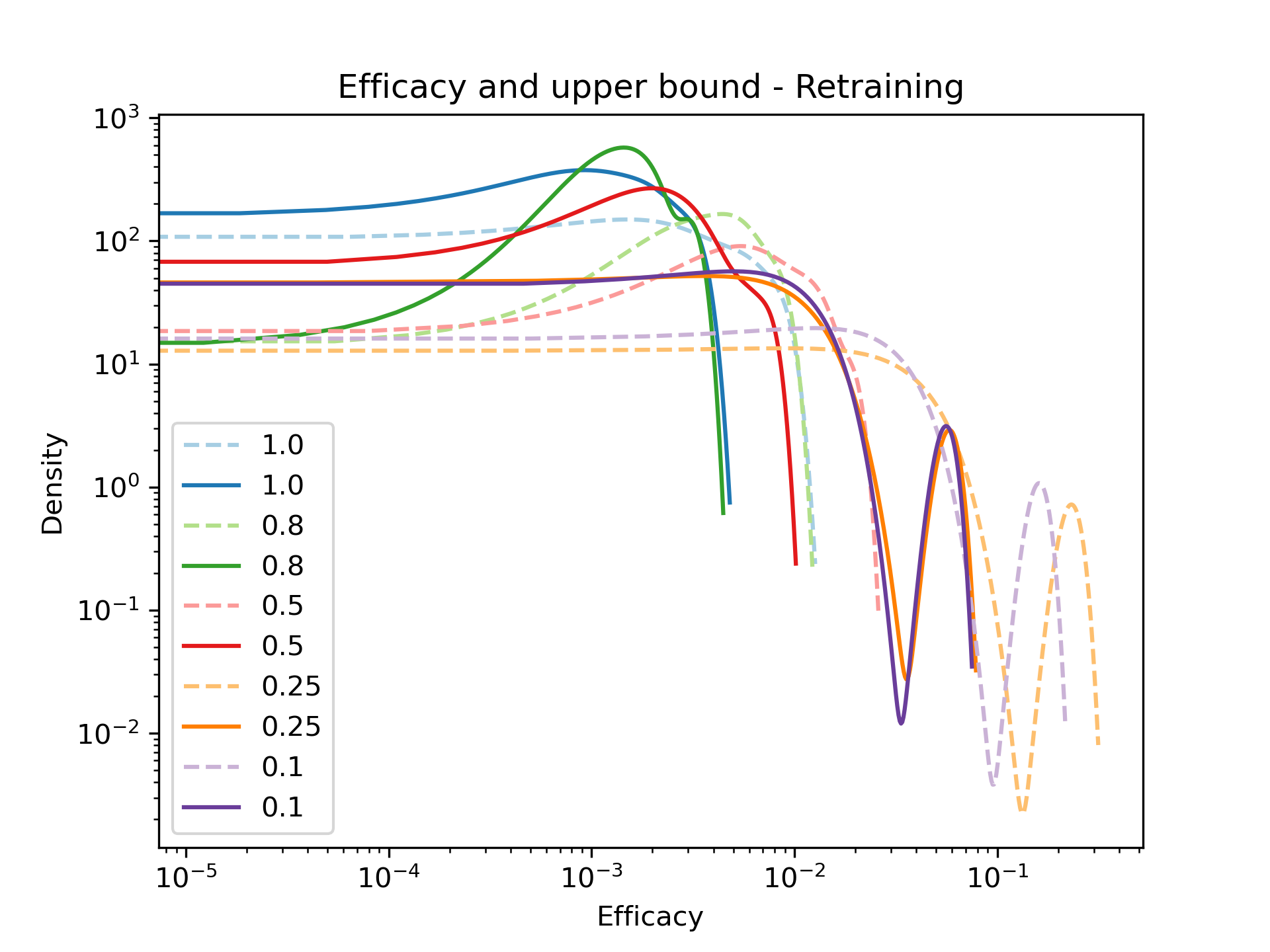}
		\caption{Retraining}
	\end{subfigure}
	\begin{subfigure}[b]{0.5\textwidth}
		\centering
		\includegraphics[scale=0.43]{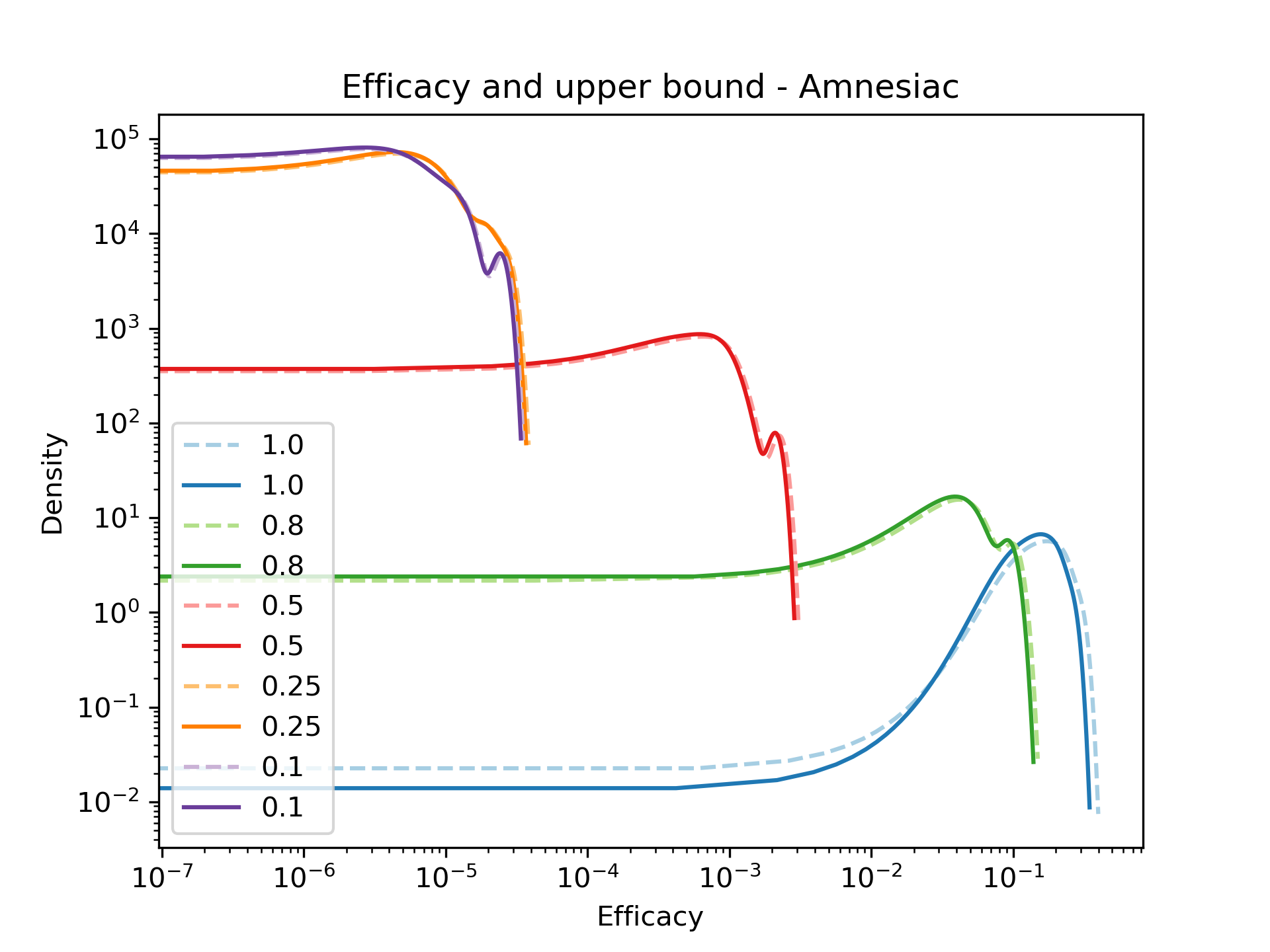}
		\caption{Amnesiac Unlearning}
	\end{subfigure}\hfill
	\begin{subfigure}[b]{0.5\textwidth}
		\centering
		\includegraphics[scale=0.43]{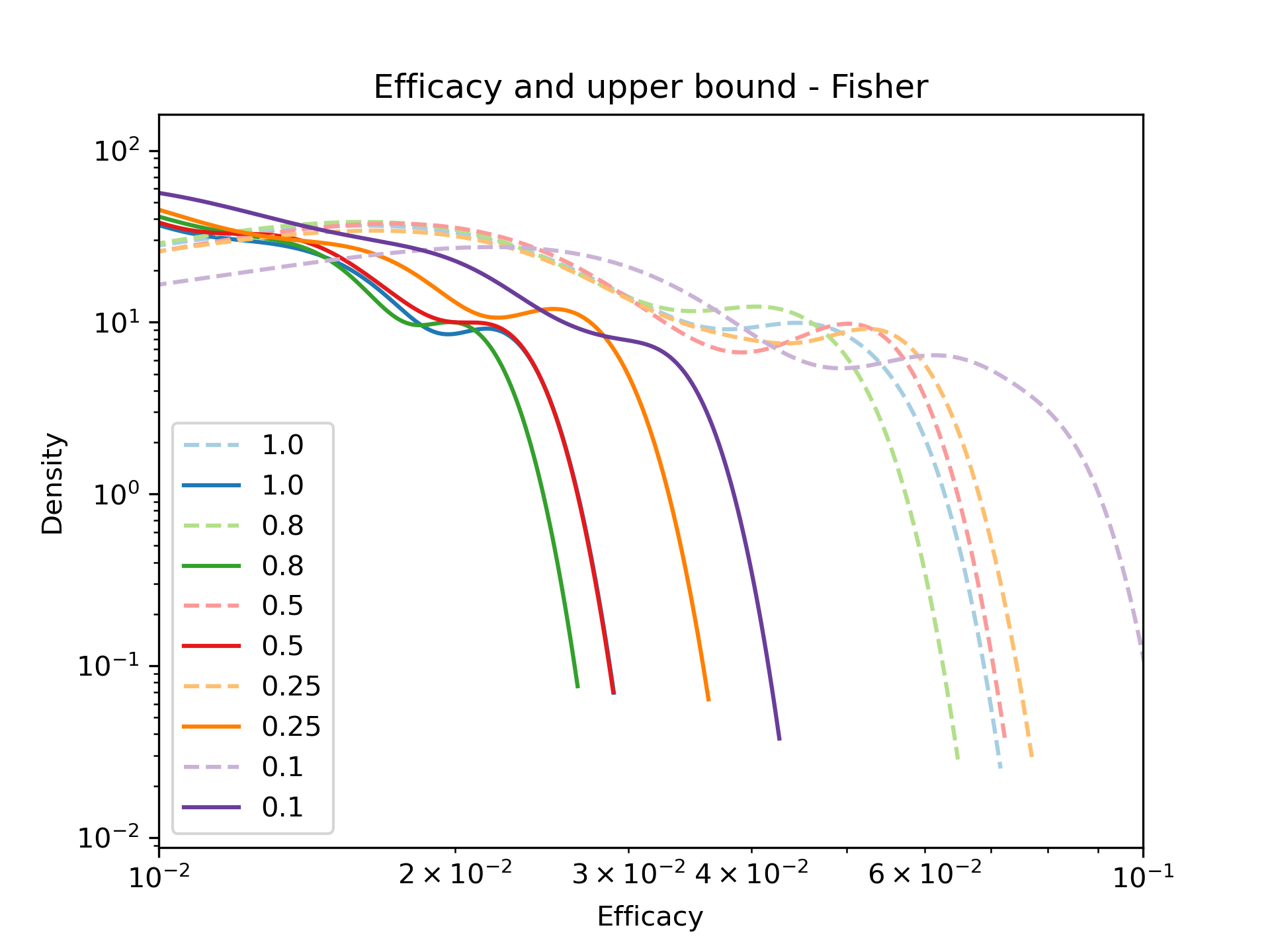}
		\caption{Fisher Forgetting}
		\label{fig:efficacy_fisher}
	\end{subfigure}
	\caption{Distributions of efficacy scores (solid lines) and upper bounds (dashed lines) over all pre-trained models trained on the CIFAR10 dataset (a) before and (b)-(d) after forgetting. Each distribution corresponds to a percentage of the target class. For reasons of readability we omit the percentage of 0.01. Both axes are log scaled.}
\end{figure}

\begin{figure}
	\centering
	\includegraphics[scale=0.6]{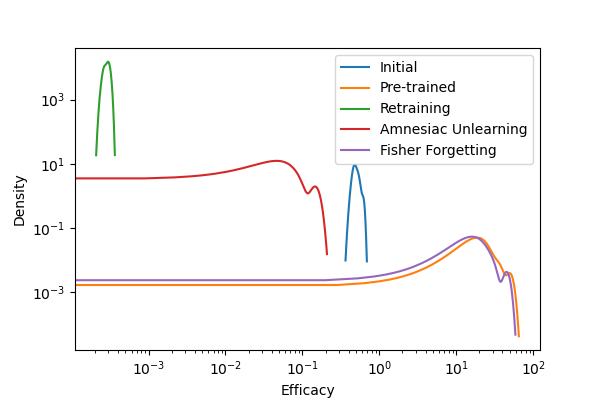}
	\caption{Efficacy comparison w.r.t. the whole target class before training (Initial), after training (Pre-trained) and after forgetting (Retraining, Amnesiac Unlearning, Fisher Forgetting). Both axes are log scaled.}
\end{figure}

\begingroup
\setlength{\tabcolsep}{10pt}
\begin{table}
	\centering
	\begin{tabular}{c c c c c}
		\toprule
		\multirow{2}{*}{Model} & \multirow{2}{*}{$p$} & \multicolumn{3}{c}{Accuracy} \\
		&& $D_r$ & $D_f$ & $D_{test}$ \\ \midrule
		\multirow{6}{*}{Pre-trained} & 1 & 0.55 ± 0.10 & 0.65 ± 0.21 & 0.26 ± 0.02 \\
		& 0.8 & 0.55 ± 0.09 & 0.64 ± 0.21 & 0.26 ± 0.02 \\
		& 0.5 & 0.55 ± 0.09 & 0.66 ± 0.22 & 0.26 ± 0.02 \\
		& 0.25 & 0.56 ± 0.08 & 0.68 ± 0.23 & 0.26 ± 0.02 \\
		& 0.1 & 0.56 ± 0.08 & 0.68 ± 0.20 & 0.26 ± 0.02 \\
		& 0.01 & 0.56 ± 0.08 & 0.95 ± 0.22 & 0.26 ± 0.02 \\ \midrule			
		\multirow{6}{*}{Retraining} & 1 & 0.24 ± 0.08 & 0.00 ± 0.00 & 0.17 ± 0.04 \\
		& 0.8 & 0.48 ± 0.10 & 0.04 ± 0.07 & 0.25 ± 0.03 \\
		& 0.5 & 0.57 ± 0.07 & 0.20 ± 0.17 & 0.27 ± 0.02 \\
		& 0.25 & 0.44 ± 0.07 & 0.31 ± 0.26 & 0.24 ± 0.02 \\
		& 0.1 & 0.54 ± 0.09 & 0.48 ± 0.24 & 0.25 ± 0.03 \\
		& 0.01 & 0.57 ± 0.07 & 0.95 ± 0.22 & 0.26 ± 0.02 \\ \midrule
		\multirow{6}{*}{Amnesiac Unlearning} & 1 & 0.11 ± 0.01 & 0.00 ± 0.00 & 0.10 ± 0.01 \\
		& 0.8 & 0.11 ± 0.01 & 0.00 ± 0.00 & 0.10 ± 0.00 \\
		& 0.5 & 0.11 ± 0.00 & 0.00 ± 0.00 & 0.10 ± 0.00 \\
		& 0.25 & 0.11 ± 0.01 & 0.00 ± 0.00 & 0.10 ± 0.01 \\
		& 0.1 & 0.10 ± 0.00 & 0.00 ± 0.00 & 0.10 ± 0.00 \\
		& 0.01 & 0.11 ± 0.01 & 0.00 ± 0.00 & 0.10 ± 0.01 \\ \midrule			
		\multirow{6}{*}{Fisher Forgetting} & 1 & 0.55 ± 0.10 & 0.65 ± 0.22 & 0.26 ± 0.02 \\
		& 0.8 & 0.55 ± 0.09 & 0.64 ± 0.22 & 0.26 ± 0.02 \\
		& 0.5 & 0.55 ± 0.09 & 0.65 ± 0.22 & 0.26 ± 0.02 \\
		& 0.25 & 0.55 ± 0.08 & 0.68 ± 0.24 & 0.26 ± 0.02 \\
		& 0.1 & 0.56 ± 0.08 & 0.69 ± 0.21 & 0.26 ± 0.02 \\
		& 0.01 & 0.56 ± 0.08 & 0.95 ± 0.22 & 0.26 ± 0.02 \\ \toprule
	\end{tabular}
	
	\caption{Mean accuracy and standard deviation of all CIFAR10 models on the remaining data $D_r$, the target data $D_f$ and the test data $D_{test}$. $p$ denotes the percentage of the target data and $*$ indicates that the accuracy values are the same over all percentages.}
	\label{tab:acc}
\end{table}
\endgroup

\begin{figure}[ht]
	\centering
	\begin{subfigure}[b]{0.5\textwidth}
		\centering
		\includegraphics[scale=0.43]{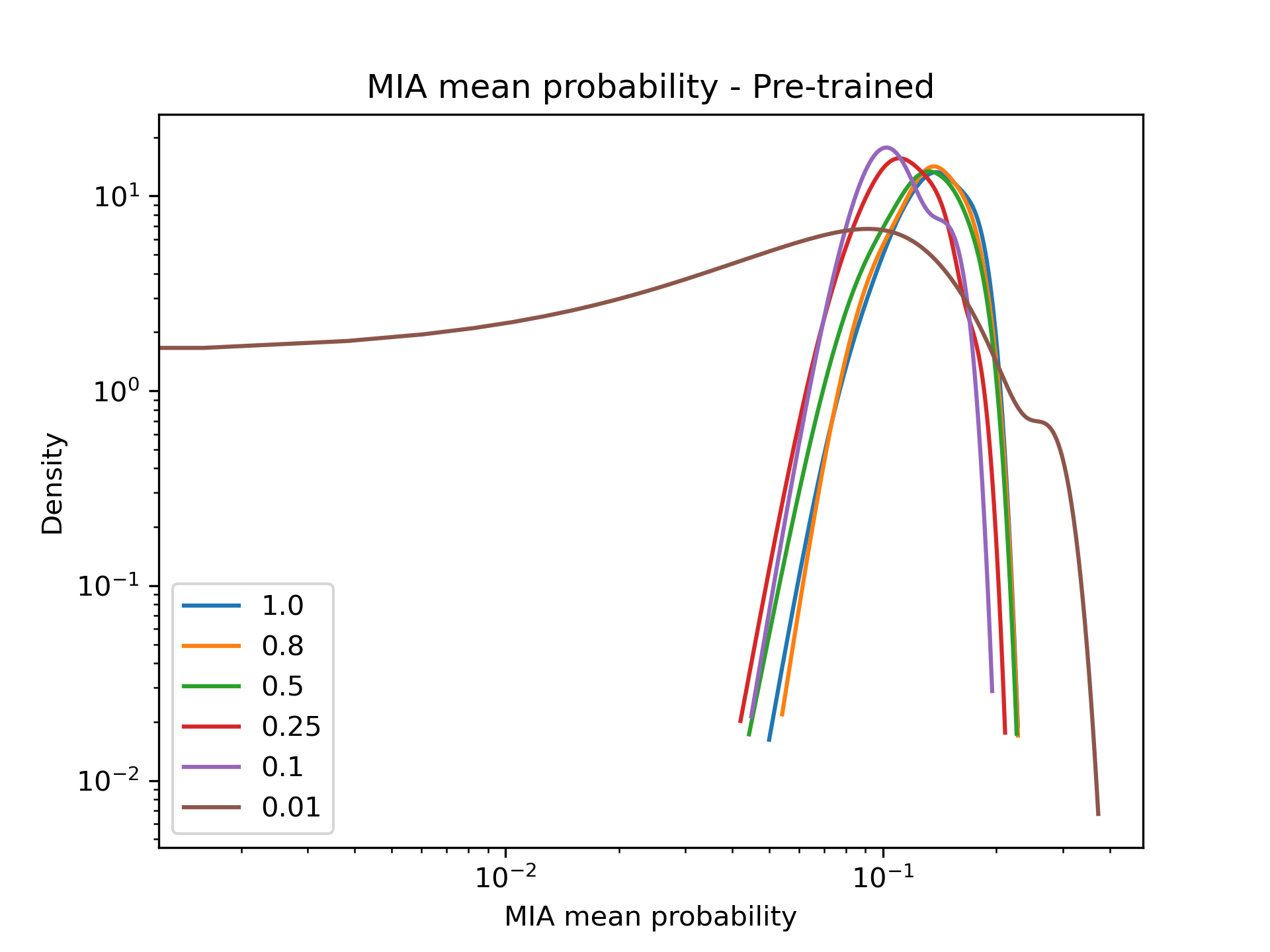}	
		\caption{Pre-trained}
	\end{subfigure}\hfill
	\begin{subfigure}[b]{0.5\textwidth}
		\centering
		\includegraphics[scale=0.43]{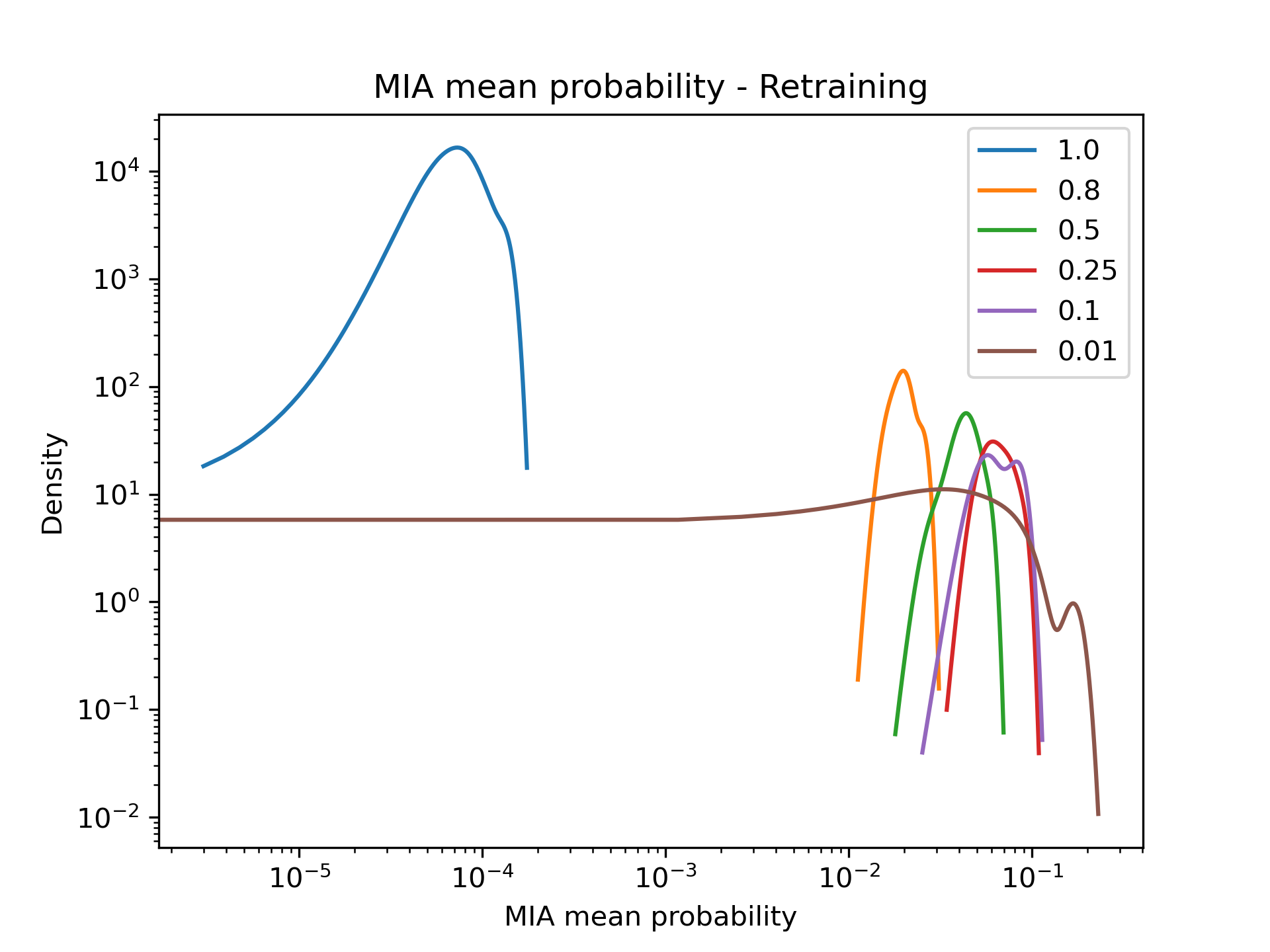}
		\caption{Retraining}
	\end{subfigure}
	\begin{subfigure}[b]{0.5\textwidth}
		\centering
		\includegraphics[scale=0.43]{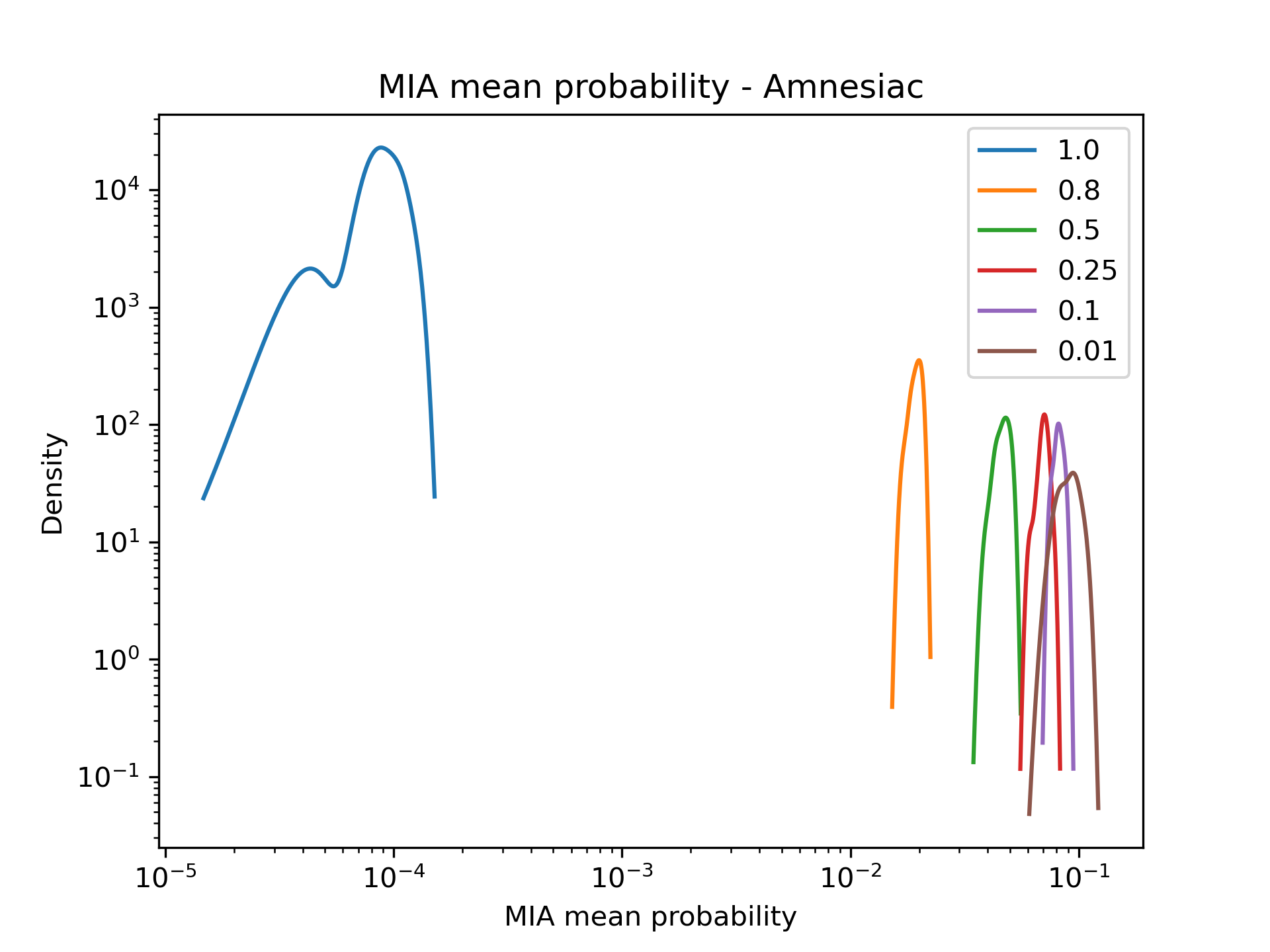}
		\caption{Amnesiac Unlearning}
	\end{subfigure}\hfill
	\begin{subfigure}[b]{0.5\textwidth}
		\centering
		\includegraphics[scale=0.43]{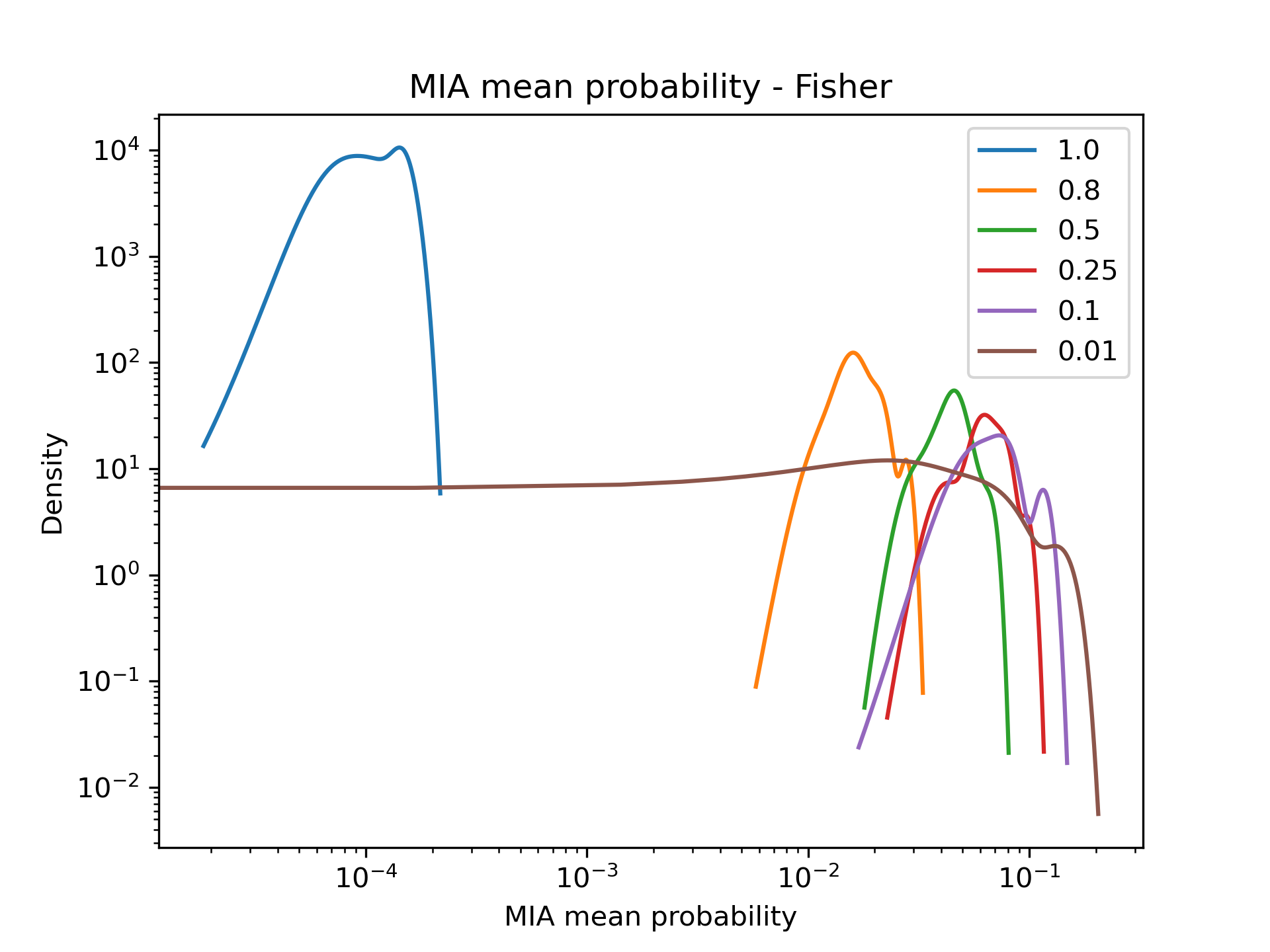}
		\caption{Fisher Forgetting}
	\end{subfigure}
	\caption{Distributions of membership inference attack mean probabilities over all pre-trained models trained on the CIFAR10 dataset (a) before and (b)-(d) after forgetting. Each distribution corresponds to a percentage of the target class. Both axes are log scaled.}
	\label{fig:mean_mia}
\end{figure}

\begin{figure}[ht]
	\centering
	\begin{subfigure}[b]{0.5\textwidth}
		\centering
		\includegraphics[scale=0.43]{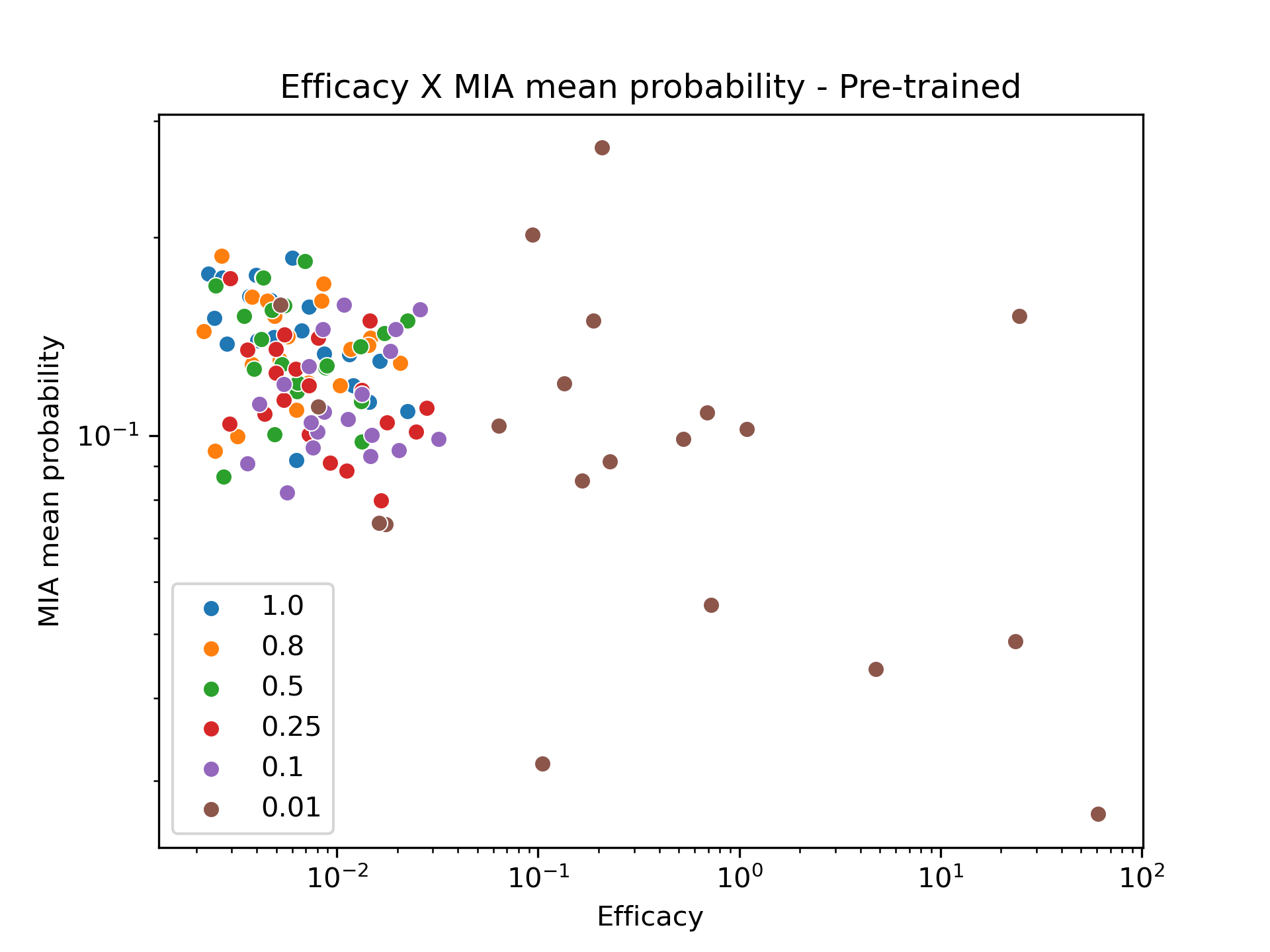}
		\caption{Pre-trained}
	\end{subfigure}\hfill
	\begin{subfigure}[b]{0.5\textwidth}
		\centering
		\includegraphics[scale=0.43]{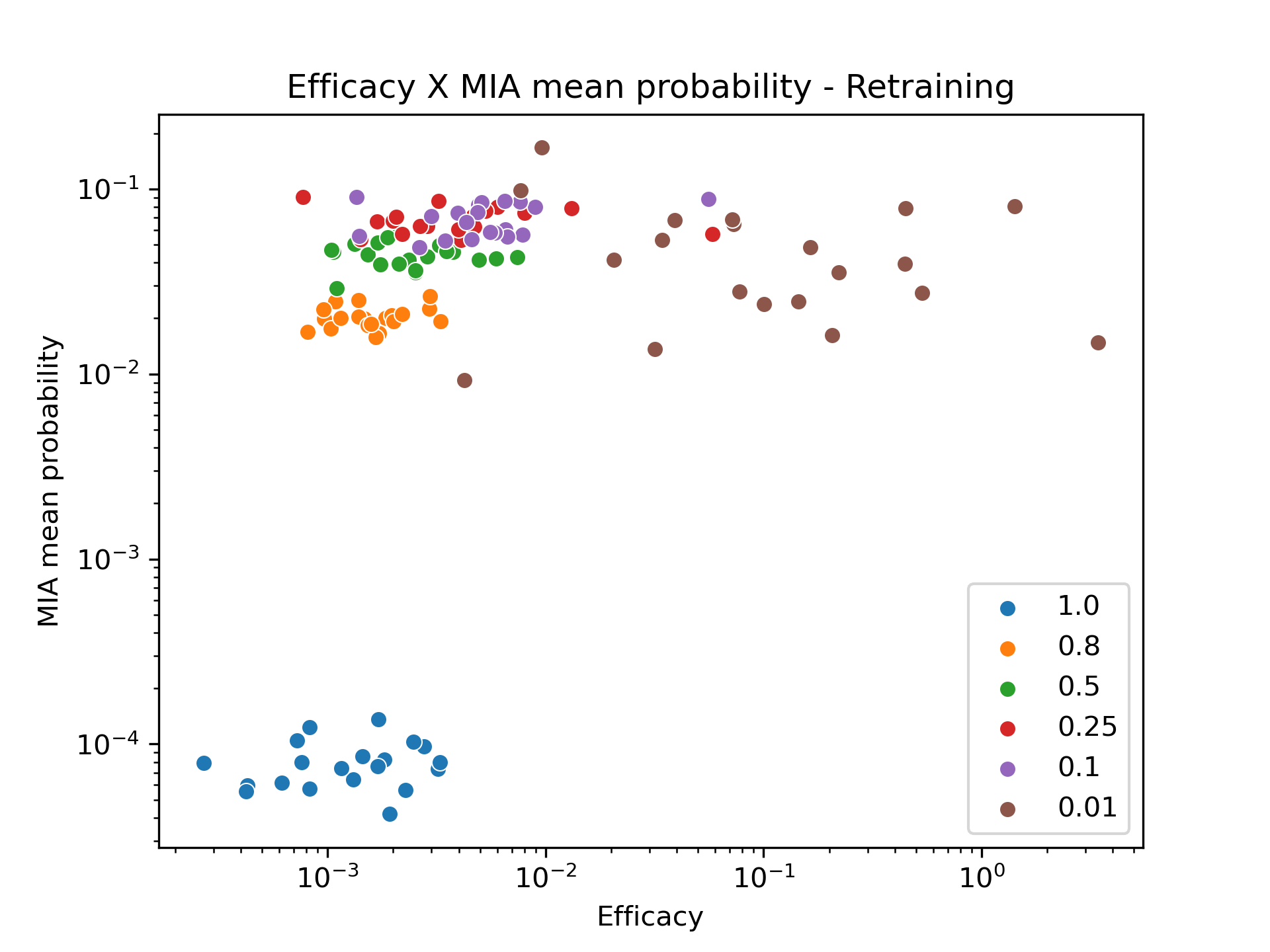}
		\caption{Retraining}
	\end{subfigure}
	\begin{subfigure}[b]{0.5\textwidth}
		\centering
		\includegraphics[scale=0.43]{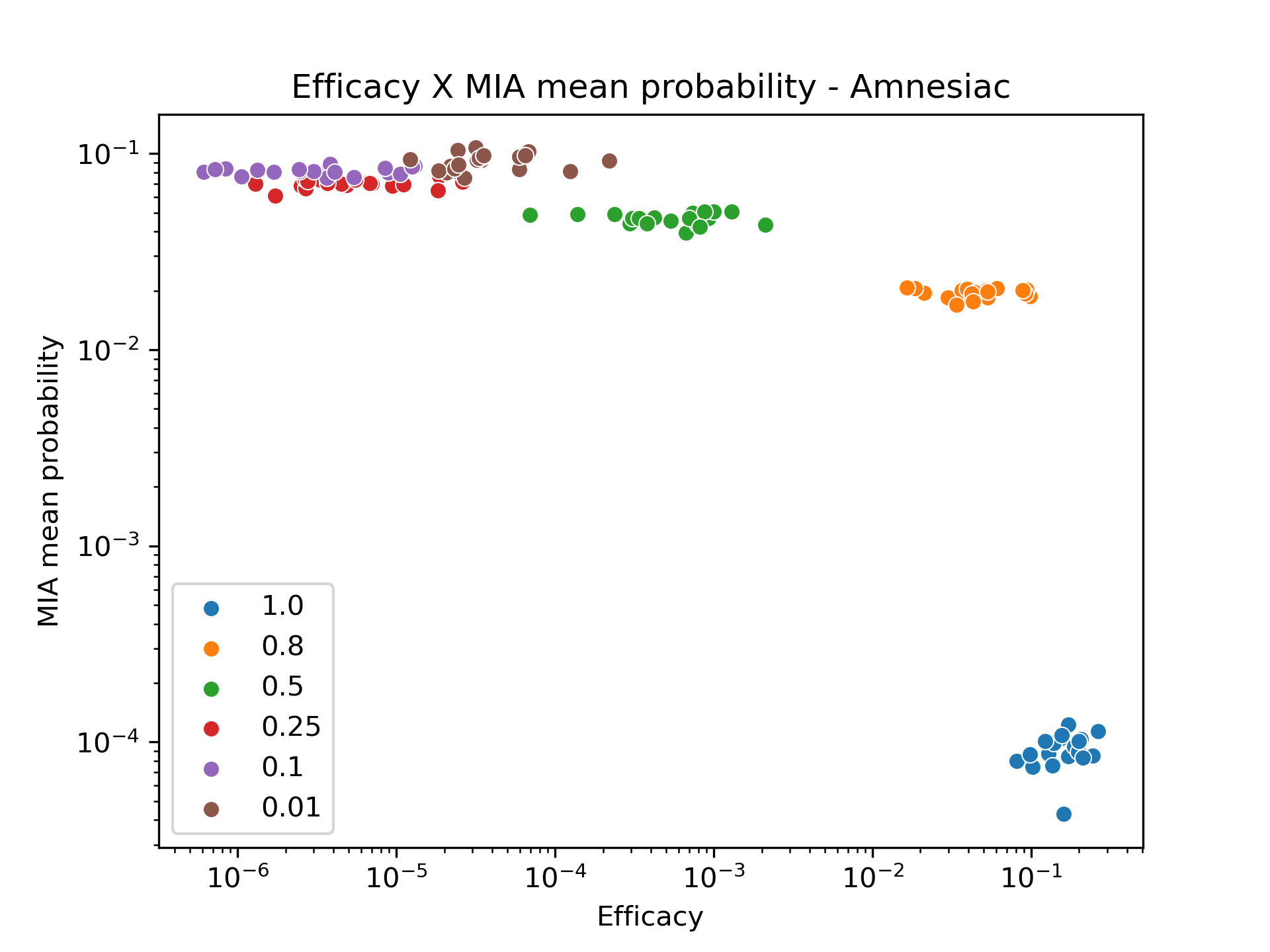}
		\caption{Amnesiac Unlearning}
	\end{subfigure}\hfill
	\begin{subfigure}[b]{0.5\textwidth}
		\centering
		\includegraphics[scale=0.43]{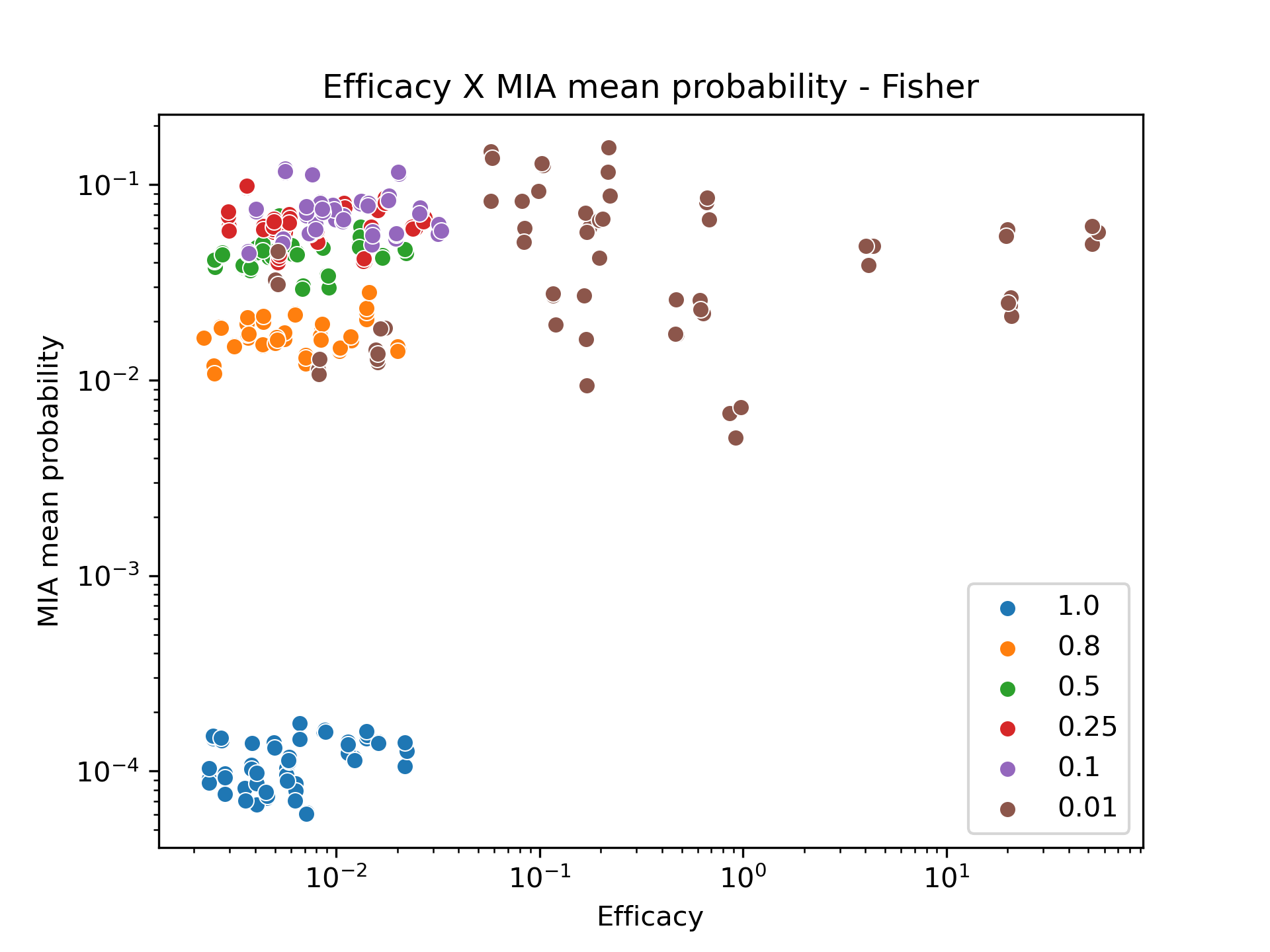}
		\caption{Fisher Forgetting}
	\end{subfigure}
	\caption{Log-log plot showing the relation between the efficacy and the membership inference attack mean probability (a) before and (b)-(d) after forgetting.}
	\label{fig:efficacy_mia}
\end{figure}

\end{document}